\def\paperTitle{On Computational Limits of FlowAR Models: Expressivity and Efficiency}

\def\paperAuthor{
Zhao Song\thanks{\texttt{magic.linuxkde@gmail.com}. UC Berkeley.}
}

\ifdefined\isarxiv
\documentclass[11pt]{article}
\usepackage[numbers]{natbib}
\else
\documentclass[twoside]{article}
\usepackage[accepted]{aistats2026}
\fi

\ifdefined\isarxiv

\usepackage{amsmath}
\usepackage{amsthm}
\usepackage{amssymb}
\usepackage{algorithm}
\usepackage{subfig}
\usepackage{algpseudocode}
\usepackage{graphicx}
\usepackage{grffile}
\usepackage{wrapfig,epsfig}
\usepackage{url}
\usepackage{xcolor}
\usepackage{epstopdf}

\usepackage{bbm}
\usepackage{dsfont}

\else 

\usepackage[round]{natbib}

\fi

\ifdefined\isarxiv

\else
\usepackage{amsmath}
\usepackage{amsthm}
\usepackage{amssymb}
\usepackage{algorithm}
\usepackage{subfig}
\usepackage{algpseudocode}
\usepackage{graphicx}
\usepackage{grffile}
\usepackage{wrapfig,epsfig}
\usepackage{url}
\usepackage{xcolor}
\usepackage{epstopdf}

\usepackage{bbm}
\usepackage{dsfont}

\usepackage{hyperref}
\fi
 
\allowdisplaybreaks

\ifdefined\isarxiv

\usepackage{tikz}
\usepackage{hyperref}  
\hypersetup{colorlinks=true,citecolor=blue,linkcolor=blue} 
\usetikzlibrary{arrows}
\usepackage[margin=1in]{geometry}
\fi
 
\graphicspath{{./figs/}}

\theoremstyle{plain}
\newtheorem{theorem}{Theorem}[section]
\newtheorem{lemma}[theorem]{Lemma}
\newtheorem{definition}[theorem]{Definition}

\newtheorem{remark}[theorem]{Remark}

\ifdefined\isarxiv

\else
\renewcommand\cite\citep
\fi

\newcommand{\wh}{\widehat}
\newcommand{\wt}{\widetilde}

\newcommand{\R}{\mathbb{R}}

\renewcommand{\d}{\mathrm{d}}

\newcommand{\X}{\mathsf{X}}
\newcommand{\Y}{\mathsf{Y}}
\newcommand{\Z}{\mathsf{Z}}

\newcommand{\F}{\mathsf{F}}
\newcommand{\V}{\mathsf{V}}

\newcommand{\TC}{\mathsf{TC}}

\newcommand{\dlogtime}{\mathsf{DLOGTIME}}

\DeclareMathOperator{\poly}{poly}

\DeclareMathOperator{\diag}{diag}

\begin{document}

\ifdefined\isarxiv

\date{}
\title{\paperTitle}
\author{\paperAuthor}

\else

%

%
\runningauthor{Cao, Gong, Ke, Li, Liang, Sha, Shi, Song}

\twocolumn[
\aistatstitle{\paperTitle}
\aistatsauthor{ 
Yang Cao$^{1}$ \And Chengyue Gong$^2$ \And Yekun Ke \And Xiaoyu Li \AND Yingyu Liang$^3$ \And Zhizhou Sha$^2$ \And Zhenmei Shi$^4$ \And Zhao Song$^5$ 
}
\aistatsaddress{ $^1$Wyoming Seminary \And $^2$UT-Austin \And $^3$The University of Hong Kong \AND $^4$University of Wisconsin-Madison \And $^5$Simons Institute, UC Berkeley \texttt{magic.linuxkde@gmail.com}}
]

\fi

\ifdefined\isarxiv
\begin{titlepage}
  \maketitle
  \begin{abstract}
    The expressive power and computational complexity of deep visual generative models, such as flow-based and autoregressive (AR) models, have gained considerable interest for their wide-ranging applications in generative tasks. However, the theoretical characterization of their expressiveness through the lens of circuit complexity remains underexplored, particularly for the state-of-the-art architecture like FlowAR proposed by [Ren et al., 2024], which integrates flow-based and autoregressive mechanisms.  This gap limits our understanding of their inherent computational limits and practical efficiency. In this study, we address this gap by analyzing the circuit complexity of the FlowAR architecture. We demonstrate that when the largest feature map produced by the FlowAR model has dimensions $n \times n \times c$, the FlowAR model is simulable by a family of threshold circuits $\mathsf{TC}^0$, which have constant depth $O(1)$ and polynomial width $\mathrm{poly}(n)$. This is the first study to rigorously highlight the limitations in the expressive power of FlowAR models. Furthermore, we identify the conditions under which the FlowAR model computations can achieve almost quadratic time. To validate our theoretical findings, we present efficient model variant constructions based on low-rank approximations that align with the derived criteria. Our work provides a foundation for future comparisons with other generative paradigms and guides the development of more efficient and expressive implementations.

  \end{abstract}
  \thispagestyle{empty}
\end{titlepage}

{\hypersetup{linkcolor=black}
\tableofcontents
}
\newpage

\else

\begin{abstract}
The expressive power and computational complexity of deep visual generative models, such as flow-based and autoregressive (AR) models, have gained considerable interest for their wide-ranging applications in generative tasks. However, the theoretical characterization of their expressiveness through the lens of circuit complexity remains underexplored, particularly for the state-of-the-art architecture like FlowAR proposed by [Ren et al., 2024], which integrates flow-based and autoregressive mechanisms.  This gap limits our understanding of their inherent computational limits and practical efficiency. In this study, we address this gap by analyzing the circuit complexity of the FlowAR architecture. We demonstrate that when the largest feature map produced by the FlowAR model has dimensions $n \times n \times c$, the FlowAR model is simulable by a family of threshold circuits $\mathsf{TC}^0$, which have constant depth $O(1)$ and polynomial width $\mathrm{poly}(n)$. This is the first study to rigorously highlight the limitations in the expressive power of FlowAR models. Furthermore, we identify the conditions under which the FlowAR model computations can achieve almost quadratic time. To validate our theoretical findings, we present efficient model variant constructions based on low-rank approximations that align with the derived criteria. Our work provides a foundation for future comparisons with other generative paradigms and guides the development of more efficient and expressive implementations.

\end{abstract}

\fi


\section{Introduction}\label{sec:introduction}

Visual generation has become a transformative force in artificial intelligence, reshaping capabilities in creative design, media synthesis, and digital content creation. Advances in deep generative models, such as Generative Adversarial Networks (GANs) \cite{gpm+20}, Variational Autoencoders (VAEs) \cite{doe16}, diffusion models \cite{hja20,ssk+20} and flow-based model \cite{kd18}, have enabled the synthesis of high-fidelity images, videos, and 3D assets with unprecedented diversity and realism. The introduction of the visual autoregressive model (VAR) \cite{tjy+24} represents a significant shift in the paradigm in the visual generation field. The VAR model adopts a coarse-to-fine Scale-wise prediction to replace the traditional autoregressive image generation techniques. This innovative technique enables the VAR model to effectively capture visual distributions while outperforming diffusion transformers in image generation benchmarks.

Recently, the introduction of FlowAR \cite{ryh+24} has further advanced the field of autoregressive visual generation. Specifically, FlowAR streamlines the scale design of VAR, improving generalization for predictions at the next scale and enabling seamless integration with the Flow Matching model \cite{lgl23} for high-quality image generation. It is worth noting that FlowAR has achieved cutting-edge results in multiple empirical studies of visual generation.

As the visual generation model architectures grow increasingly sophisticated to meet the demands of high-resolution and photorealistic generation, critical questions arise regarding their computational efficiency and intrinsic expressive power. While empirical improvements in generation quality dominate the current discourse, comprehending the theoretical foundations of these models continues to be a pivotal challenge. To tackle the challenge mentioned above, some prior researchers have made significant contributions. For example, \cite{ms24} show that $\mathsf{DLOGTIME}$-uniform $\TC^0$ circuits can simulate softmax-attention transformers; later, \cite{cll+24} show that the introduction of RoPE will not enhance the express power of transformer; \cite{kll+25_circuit_var} present the circuit complexity for the VAR model. Up to now, the expressiveness from a circuit complexity perspective of the FlowAR model remains unexplored. This gap raises an important question:
\begin{center}
   {\it Does the Flow Matching architecture enhance the expressive power of the VAR Model?} 
\end{center}
This study seeks to explore this question through the lens of circuit complexity. First, we provide a model formulation for each module of FlowAR. Our insight is that using circuit complexity theory, we prove that each module of FlowAR, including the Attention Layer, Flow-Matching Layer, and others, can be simulated by a constant-depth, polynomial-size $\TC^0$ circuit. Ultimately, the combined result shows that the entire FlowAR architecture can be simulated by a constant-depth, polynomial-size $\TC^0$ circuit. Therefore, our conclusion offers a negative response to the question: despite the inclusion of the flow-matching mechanism, the expressive power of FlowAR, in terms of circuit complexity, is on par with that of the VAR model.

In addition, we explored the runtime of the FlowAR model inference process and potential efficient algorithms. Specifically, we analyzed the runtime of each module in the FlowAR model and found that the bottleneck affecting the overall runtime originates from the computation of the attention mechanism. As a result, we accelerated the original attention computation using low-rank approximation, which makes the overall runtime of the FlowAR model almost quadratic.

The primary contributions of our work are summarized below:
\begin{itemize}
    \item {\bf Circuit Complexity Bound:} FlowAR model can be simulated by a $\mathsf{DLOGTIME}$-uniform $\mathsf{TC}^0$ family. (Theorem~\ref{thm:flowar_tc0})
    \item {\bf Provably Efficient Criteria:} Suppose the largest feature map produced by the FlowAR model has dimensions $n \times n \times c$ and $c = O(\log n)$. We prove that the time complexity of the FlowAR model architecture is $
    O(n^{4+o(1)})$. By applying low-rank approximation to the Attention module within FlowAR, we obtain a FlowAR model with an almost quadratic runtime. Explicitly, we demonstrate that the FlowAR model variant's time complexity in realistic settings is $O(n^{2+o(1)})$. (Theorem~\ref{thm:upper_bound:formal})
\end{itemize}

{\bf Roadmap.} 
The paper's organizational structure is outlined as follows: Section~\ref{sec:related_work} synthesizes key academic contributions in the domain. Section~\ref{sec:preliminary} then elucidates foundational circuit complexity principles essential for subsequent analysis. Subsequent sections progress systematically, with Section~\ref{sec:model_formulation_of_flowar} detailing mathematical formalizations for all FlowAR modules. Section~\ref{sec:main_result} outlines our principal findings. 
Section~\ref{sec:efficient_critieria} presents provably efficient criteria of the fast FlowAR model.
In Section~\ref{sec:conclusion}, we conclude our paper.

\section{Related Work}\label{sec:related_work}

\paragraph{Flow-based and diffusion-based models.} Another line of work focuses on flow-based and diffusion-based models for image and video generation \cite{hja20,hhs23,ltl+24}. The latent diffusion model (LDM) \cite{rbl+22} transforms image generation from pixel space to latent space, reducing the computational cost of diffusion-based generative models. This transformation enables these models to scale to larger datasets and model parameters, contributing to the success of LDM. Subsequent works, such as  U-ViT \cite{bnx+23} and  DiT \cite{px23}, replace the U-Net architecture with Vision Transformers (ViT) \cite{d20}, leveraging the power of Transformer architectures for image generation. Later models like SiT \cite{aak21} incorporate flow-matching into the diffusion process, further enhancing image generation quality. Many later studies \cite{ekb+24,jsl+24,wsd+24,wcz+23,wxz+24} have pursued the approach of integrating the strengths of both flow-matching and diffusion models to develop more effective image generation techniques. More related works on flow models and diffusion models can be found in \cite{hst+22,swyy23,lssz24_gm,llss24,hwl+24}.

\paragraph{Circuit complexity.} Circuit complexity is a key field in theoretical computer science that explores the computational power of Boolean circuit families. Different circuit complexity classes are used to study machine learning models, aiming to reveal their computational constraints. A significant result related to machine learning is the inclusion chain $\mathsf{AC}^0 \subset \mathsf{TC}^0 \subseteq \mathsf{NC}^1$, although it is still unresolved whether $\mathsf{TC}^0 = \mathsf{NC}^1$ \cite{v99,ab09}. 
The analysis of circuit complexity limitations has served as a valuable methodology for evaluating the computational capabilities of diverse neural network structures. Recent investigations have particularly focused on Transformers and their two principal derivatives: Average-Head Attention Transformers (AHATs) and SoftMax-Attention Transformers (SMATs). Research has established that non-uniform threshold circuits operating at constant depth (within $\mathsf{TC}^0$ complexity class) can effectively simulate AHAT implementations \cite{mss22}, with parallel studies demonstrating similar computational efficiency achieved through L-uniform simulations for SMAT architectures \cite{lag+22}. Subsequent theoretical developments have extended these investigations, confirming that both architectural variants can be effectively approximated using \textsf{DLOGTIME}-uniform $\mathsf{TC}^0$ circuit models \cite{ms24}.
In addition to standard Transformers, circuit complexity analysis has also been applied to various other frameworks \cite{cll+24_mamba_circut,kll+25_circuit_var}. Other works related to circuit complexity can be referenced in \cite{cll+24,cll+24_tensor_tc}.

\section{Preliminary}\label{sec:preliminary}
All notations employed throughout this paper are present in Section~\ref{sub:notations}. Section~\ref{sec:pre:circuit} introduces circuit complexity axioms. In Section~\ref{sec:pre:float}, we define floating-point numbers and establish the complexity bounds of their operations.

\subsection{Notations}\label{sub:notations}
Given a matrix $X \in \R^{hw \times d}$, we denote its tensorized form as $\X \in \R^{h \times w \times d}$. Additionally, we define the set $[n]$ to represent $\{1,2,\cdots, n\}$ for any positive integer $n$. We define the set of natural numbers as $\mathbb{N}:= \{0,1,2,\dots\}$. Let $X \in \mathbb{R}^{m \times n}$ be a matrix, where $X_{i,j}$ refers to the element at the $i$-th row and $j$-th column. When $x_i$ belongs to $\{ 0,1 \}^*$, it signifies a binary number with arbitrary length. In a general setting, $x_i$ represents a length $p$ binary string, with each bit taking a value of either 1 or 0. Given a matrix $X \in \R^{n \times d}$, we define $\|X\|_\infty  $ as the maximum norm of $X$. Specifically, $\|X\|_\infty = \max_{i,j} |X_{i,j}|$.

\subsection{Circuit Complexity Class}\label{sec:pre:circuit}
Firstly, we present the definition of the boolean circuit.

\begin{definition}[Boolean Circuit, \cite{ab09}] 
A Boolean circuit $C_n:\{0,1\}^n \to \{0, 1\}$ is formally specified through a directed acyclic graph (DAG) where:
    Part 1.
    Nodes represent logic gates from the basis $\{\mathsf{AND},\mathsf{OR},\mathsf{NOT}\}$.
    Part 2.
    Source nodes (in degree $0$) correspond to input Boolean variables ${x_1, \dots, x_n}$.
    Part 3.
    Each non-source gate computes its output by applying its designated Boolean operation to values received via incoming edges.
\end{definition}

Then, we proceed to show the definition of languages related to a specific Boolean circuit.

\begin{definition}[Languages, page 103 of~\cite{ab09}]
A language $L \subseteq \{0, 1\}^*$ is recognized by a Boolean circuit family $\mathcal{C} = \{C_n\}_{n \in \mathbb{N}}$ if:
\begin{itemize}
    \item The family is parameterized by input length: $C_n$ operates on $n$ Boolean variables.
    \item Membership equivalence: $\forall x \in \{0,1\}^*, C_{|x|}(x) = 1\Longleftrightarrow x \in L$.
    \item Circuit existence: For every string length $n \in \mathbb{N}$, $\mathcal{C}$ contains an appropriate circuit $C_n$.
\end{itemize}
\end{definition}

Then, we present different language classes that can be recognized by different circuit families. Firstly, we introduce the $\mathsf{NC}^i$ class.
\begin{definition}[$\mathsf{NC}^i$ Complexity Class, \cite{ab09}]
    The complexity class $\mathsf{NC}^i$ comprises all languages recognized by Boolean circuit families $\{C_n\}$ satisfying:
        $\mathsf{Size}(C_n) = O(\poly(n))$.
        $\mathsf{Depth}(C_n) = O((\log n)^i)$.
        Gate constraints: (1) $\mathsf{AND}, \mathsf{OR}$ gates have bounded fan-in (2) $\mathsf{NOT}$ gates have unit fan-in.
\end{definition}

$\mathsf{AC}^i$ circuits relax the gate fan-in restriction of $\mathsf{NC}^i$ circuits. We present the definition of $\mathsf{AC}^i$ as the following:
\begin{definition}[$\mathsf{AC}^i$ Complexity Class, \cite{ab09}]
    The complexity class $\mathsf{AC}^i$ comprises all languages recognized by Boolean circuit families $\{C_n\}$ satisfying:
        Part 1.
        $\mathsf{Size}(C_n) = O(\poly(n))$.
        Part 2.
        $\mathsf{Depth}(C_n) = O((\log n)^i)$.
        Part 3.
        Gate constraints: (1) $\mathsf{AND}, \mathsf{OR}$ gates have un-bounded fan-in (2) $\mathsf{NOT}$ gates have unit fan-in.
\end{definition}

$\mathsf{TC}^i$ introduces the $\mathsf{MAJORITY}$ gate on top of $\mathsf{AC}^i$. The $\mathsf{MAJORITY}$ gate outputs $1$ if more than half of its inputs are $1$ and outputs $0$ otherwise.

\begin{definition}[$\mathsf{TC}^i$ Complexity Class, \cite{vol99}]\label{def:tc}
    The complexity class $\mathsf{TC}^i$ comprises all languages recognized by Boolean circuit families $\{C_n\}$ satisfying:
        Part 1.
        $\mathsf{Size}(C_n) = O(\poly(n))$.
        Part 2.
        $\mathsf{Depth}(C_n) = O((\log n)^i)$.
        Part 3.
        Gate constraints: (1) $\mathsf{AND}, \mathsf{OR}, \mathsf{MAJORITY}$ gates have un-bounded fan-in (3) $\mathsf{NOT}$ gates have unit fan-in.  
\end{definition}
In this paper, a boolean circuit that employs $\mathsf{MAJORITY}$ gate is referred to as a threshold circuit.

Then, we present two import definitions $\mathsf{L}$-uniformity and $\mathsf{DLOGTIME}$-uniformity.
\begin{definition}[$\mathsf{L}$-uniformity, Definition 6.5 on page 104 of~\cite{ab09}]
    Let $\mathcal{C}$ denote a circuit family (e.g., $\mathsf{NC}^i,\mathsf{AC}^i,\mathsf{TC}^i$) that decides a language $\mathsf{C}$. A language $L \subseteq \{0,1\}^*$ belongs to $\mathsf{L}$-uniform $\mathsf{C}$ if there exists a deterministic Turing machine that, on input ${1}^n$, outputs a circuit $C_n \in \mathcal{C}$ in $O(\log n)$ space, such that $C_n$ recognizes $L$.
\end{definition}
And we present the definition of $\dlogtime$-uniformity.
\begin{definition}[$\dlogtime$-uniformity, \cite{bi94}]
    Let $\mathcal{C}$ denote a circuit family (e.g., $\mathsf{NC}^i,\mathsf{AC}^i,\mathsf{TC}^i$) that decides a language $\mathsf{C}$. A language $L \subseteq \{0,1\}^*$ belongs to $\dlogtime$-uniform $\mathsf{C}$ if there exists a random access Turing machine that, on input ${1}^n$, outputs a circuit $C_n \in  \mathcal{C}$ in $O(\log n)$ time such that $C_n$ recognizes $L$.   
\end{definition}

\subsection{Circuit Complexity for Floating-Point Operations}\label{sec:pre:float}
In this section, we first introduce the key definitions of floating-point numbers.
\begin{definition}[Float point number, \cite{chi24}]
    A $p$-bit floating-point number is a tuple $\langle a,b \rangle$ of integers where
        Part 1.
        The significand $a$ satisfies $a \in (-2^p,-2^{p-1}] \cup \{0\} \cup [2^{p-1},2^p)$.
        Part 2.
        The exponent $b$ lies in the interval $b\in [-2^p, 2^p)$.
    This represents the real number $a \cdot w^b$. The set of all $p$-bit floating-point numbers is denoted $\mathsf{F}_p$.
\end{definition}

Then, we can show the circuit complexity bounds of some float point number operations.
\begin{lemma}[Float point number operations in $\TC^0$, \cite{chi24}]\label{lem:float_operations_TC} 
Assume the precision $p \leq \poly(n)$. The following hold:
    {\bf Basic Arithmetic:} Addition, comparison and multiplication of two $p$-bit floating-point numbers are computable by uniform $\mathcal{TC}^0$ circuits (depth $O(1)$, size $\poly(n)$). Denote by $d_{\mathrm{std}}$ the circuit depth for these operations.
    {\bf Iterated Multiplication:} The product of $n$ $p$-bit floating-point numbers is computable by uniform $\TC^0$ circuits (depth $O(1)$, size $\poly(n)$). Denote by $d_{\otimes}$ the required circuit depth.
    {\bf Iterated Addition:} The sum of $n$ $p$-bit floating-point numbers is computable by uniform $\TC^0$ circuits (depth $O(1)$, size $\poly(n)$). Denote by $d_{\oplus}$ the required circuit depth.
\end{lemma}

\begin{lemma}[Exponential Approximation in $\TC^0$, \cite{chi24}]\label{lem:exp}
    Let precision $p \leq \poly(n)$. For every $p$-bit floating-point number $x \in \mathsf{F}_p$, there exists a constant depth uniform $\mathsf{TC}^0$ circuit of size $\poly(n)$ that can compute $\exp(x)$ with relative error bounded by $2^{-p}$. Denote by $d_{\exp}$ the required circuit depth.

\end{lemma}

\begin{lemma}[Square Root Approximation in $\TC^0$, \cite{chi24}]\label{lem:sqrt}
    Let precision $p \leq \poly(n)$. For every $p$-bit floating-point number $x \in \mathsf{F}_p$, there exists a constant depth uniform $\mathsf{TC}^0$ circuit of size $\poly(n)$ that can compute $\sqrt{x}$ with relative error bounded by $2^{-p}$. Denote by $d_{\mathrm{sqrt}}$ the required circuit depth.    
\end{lemma}

\section{Complexity of FlowAR Architecture}\label{sec:main_result}
This section presents key results on the circuit complexity of fundamental modules in the FlowAR architecture. Section~\ref{lem:matrix_multi} analyzes matrix multiplication, while Section~\ref{sec:down_up_tc0} examines the up-sampling and down-sampling functions. In Sections~\ref{sec:mlp_tc0} and \ref{sec:ffn_tc0}, we compute the circuit complexity of the MLP and FFN layers, respectively. Sections~\ref{sec:attention_tc0} and \ref{sec:ln_tc0} focus on the single attention layer and layer normalization. Section~\ref{sec:flow_matching} addresses the flow-matching layer. Finally, Section~\ref{sec:main_result} presents our main result, establishing the circuit complexity bound for the complete FlowAR architecture.
\subsection{Computing Matrix Products in \texorpdfstring{$\TC^0$}{}}\label{sec:compute_matrix_product}
we demonstrate that matrix multiplication is computable in $\mathsf{TC}^0$, which will be used later.
\begin{lemma}[Matrix multiplication in $\TC^0$, \cite{cll+24}]\label{lem:matrix_multi}
    Let the precision $p \leq \poly(n)$. Let $X \in \mathsf{F}_p^{n_1 \times d}, Y \in \mathsf{F}_p^{d \times n_2}$ be matrices. Assume $n_1\leq\poly(n), n_2\leq\poly(n)$. The matrix product $XY$ can be computed by a uniform $\mathsf{TC}^0$ circuit with:
    \begin{itemize}
        \item 
        Size: $\poly(n)$.
        \item 
        Depth: $d_{\mathrm{std}}+d_{\oplus}$.
        , where $d_{\mathrm{std}}$ and $d_{\oplus}$ are defined in Definition~\ref{lem:float_operations_TC}.
    \end{itemize}
    
\end{lemma}

\subsection{Computing Down-Sampling and Up-Sampling in in \texorpdfstring{$\TC^0$}{}}\label{sec:down_up_tc0}
In this section, we show that Up-Sampling can be efficiently computable by a uniform $\mathsf{TC}^0$ circuit.
\begin{lemma}[Up-Sampling computation in $\mathsf{TC}^0$]\label{lem:up_tc0}
     Let $\X \in \R^{h \times w \times c}$ be the input tensor. Let $\phi_{\mathrm{up}}(X,r):\R^{h \times w \times c} \to \R^{(hr) \times (wr) \times c}$ denote the bicubic up sample function defined in Definition~\ref{def:bicubic_up_sample_function}. Assume $n = h = w$. Assume $r \leq n$. Assume $c \leq n$. Assume $p \leq \poly(n)$.
    The linear up sample function can be computed by a uniform $\mathsf{TC}^0$ circuit with:
    \begin{itemize}
        \item Size: $\poly(n)$.
        \item Depth: $2d_\mathrm{std} + d_{\oplus}$.
        , where $d_{\mathrm{std}}$ and $d_{\oplus}$ are defined in Definition~\ref{lem:float_operations_TC}.
    \end{itemize}
\end{lemma}
\begin{proof}
    For each $i \in [nr], j \in [nr], l \in [c]$, we need to compute $\phi_{\mathrm{up}}(\X,r)_{i,j,l} = \sum_{s=-1}^2 \sum_{t=-1}^2 W(s) \cdot W(t) \cdot \X_{\frac{i}{r}+s,\frac{j}{r}+s,l}$.  We need a $2d_{\mathrm{std}}$ depth and $\poly(n)$ size circuit to compute $W(s)\cdot W(t)\cdot \X_{\frac{i}{r}+s,\frac{j}{r}+s,l}$ by Part 1 of Lemma~\ref{lem:float_operations_TC} and for all $s,t \in \{-1,0,1,2\}$, these terms can be computed in parallel. Furthermore, by Part 3 of Lemma~\ref{lem:float_operations_TC}, we can need a $d_{\oplus}$ depth and $\poly(n)$ size circuit to compute $\sum_{s=-1}^2 \sum_{t=-1}^2 W(s)\cdot W(t)\cdot \X_{\frac{i}{r}+s,\frac{j}{r}+s,l}$. Since the computation of $\phi_{\mathrm{up}}(\X,r)_{i,j,l}$ needs a $2d_{\mathrm{std}}+d_{\oplus}$ depth and $\poly(n)$ size circuit.

    Since for all $i \in [nr], j \in [nr], l \in [c]$, we can compute $\phi_{\mathrm{up}}(\X,r)_{i,j,l}$ in parallel, then the total depth of the circuit is $2d_{\mathrm{std}} + d_{\oplus}$ and size remains $\poly(n)$.
\end{proof}

Then, we move forward to consider the down-sampling function.
\begin{lemma}[Down-Sampling computation in $\mathsf{TC}^0$]\label{lem:down_tc0}
    Let $\X \in \R^{h \times w \times c}$ be the input tensor. Let $\phi_{\mathrm{down}}(X,r)$ be the linear down sample function from Definition~\ref{def:linear_down_sample_function}. Assume $n = h = w$. Assume $r \leq n$. Assume $c \leq n$. Assume $p \leq \poly(n)$.

    The function $\phi_{\mathrm{down}}$ can be computed by a uniform $\mathsf{TC}^0$ circuit with:
    \begin{itemize}
        \item Size: $\poly(n)$.
        \item Depth: $d_\mathrm{std} + d_{\oplus}$.
        , where $d_{\mathrm{std}}$ and $d_{\oplus}$ are defined in Definition~\ref{lem:float_operations_TC}.
    \end{itemize}
\end{lemma}
\begin{proof}
    By Definition~\ref{def:linear_down_sample_function}, we know that down-sampling computation is essentially matrix multiplication. Then, by Lemma~\ref{lem:matrix_multi}, we can easily get the proof.
\end{proof}

\subsection{Computing Multiple-layer Perceptron in \texorpdfstring{$\TC^0$}{} }\label{sec:mlp_tc0}
 We prove that MLP computation can be efficiently simulated by a uniform $\mathsf{TC}^0$ circuit.
\begin{lemma}[MLP computation in $\mathsf{TC}^0$, informal version of Lemma~\ref{lem:mlp_tc0_formal}]\label{lem:mlp_tc0_informal}
    Given an input tensor $\X\in \R^{h \times w \times c}$. Let $\mathsf{MLP}(\X,c,d)$ be the MLP layer defined in Definition~\ref{def:mlp}. Under the following constraints:
        {\bf Part 1.}
        Satisfy $h = w = n$,
        {\bf Part 2.}
        Channel bounds: $c, d\leq n$,
        {\bf Part 3.}
        Precision: $p \leq \poly(n)$,
    The $\mathsf{MLP}(\X,c,d)$ function can be computed by a uniform $\mathsf{TC}^0$ circuit with:
    \begin{itemize}
        \item 
        Size: $\poly(n)$.
        \item
        Depth: $2d_\mathrm{std} + d_{\oplus}$, where $d_{\mathrm{std}}$ and $d_{\oplus}$ are defined in Definition~\ref{lem:float_operations_TC}.
    \end{itemize}
\end{lemma}

\subsection{Computing Feed-Forward Layer in \texorpdfstring{$\TC^0$}{} }\label{sec:ffn_tc0}
We also prove that feed-forward network computation can be simulated by a uniform $\mathsf{TC}^0$ circuit.
\begin{lemma}[FFN computation in $\mathsf{TC}^0$, informal version of Lemma~\ref{lem:ffn_tc0_formal}]\label{lem:ffn_tc0_informal}
Given an input tensor $\X\in \R^{h \times w \times c}$. Let $\mathsf{FFN}(X):\R^{h \times w \times c} \to \R^{h \times w \times c}$ as defined in Definition~\ref{def:ffn}. Under the following constraints:
   {\bf Satisfy} $h = w = n$.
   {\bf Channel bound}: $c \leq n$.
   {\bf Precision bound}: $p \leq \poly(n)$.

The $\mathsf{FFN}(\X)$ layer can be computed by a uniform $\mathsf{TC}^0$ circuit with:
\begin{itemize}
    \item 
    Size: $\poly(n)$.
    \item 
    Depth: $6d_\mathrm{std} + 2d_{\oplus}$.
\end{itemize}
\end{lemma}

\subsection{Computing Single Attention Layer in \texorpdfstring{$\TC^0$}{}.} \label{sec:attention_tc0}
 We prove the single attention layer can be efficiently simulated by a uniform $\mathsf{TC}^0$ circuit.
\begin{lemma}[Attention layer computation in $\mathsf{TC}^0$, informal version of Lemma~\ref{lem:attn_tc0_formal}]\label{lem:attn_tc0_informal}
     Given an input tensor $\X \in \R^{h \times w \times c}$. Let $\mathsf{Attn}(X):\R^{h \times w \times c} \to \R^{h \times w \times c}$ as defined in Definition~\ref{def:attn_layer}. Under the following constraints:
       {\bf Satisfy} $h = w = n$.
       {\bf Channel bound}: $c \leq n$.
       {\bf Precision bound}: $p \leq \poly(n)$.
     The $\mathsf{Attn}(\X)$ layer can be computed by a uniform $\mathsf{TC}^0$ circuit with:
     \begin{itemize}
        \item 
        Size: $\poly(n)$.
        \item 
        Depth: $6(d_{\mathrm{std}} + d_{\oplus}) + d_{\exp}$.
     \end{itemize}
\end{lemma}

\subsection{Computing Layer-wise Norm Layer in \texorpdfstring{$\TC^0$}{}.}\label{sec:ln_tc0}
We prove that the layer normalization layer can be efficiently simulated by a uniform $\mathsf{TC}^0$ circuit.
\begin{lemma}[Layer normalization layer computation in $\TC^0$, informal version of Lemma~\ref{lem:ln_tc0_formal}]\label{lem:ln_tc0_informal}
    Given an input tensor $\X \in \R^{h \times w \times c}$. Let $\mathsf{LN}(X):\R^{h \times w \times c} \to \R^{h \times w \times c}$ as defined in Definition~\ref{def:ln}. Under the following constraints:
         {\bf Part 1.}
         Satisfy $h = w = n$,
         {\bf Part 2.}
         Channel bound: $c \leq n$,
         {\bf Part 3.}
         Precision bound: $p \leq \poly(n)$.
     
     The $\mathsf{LN}(\X)$ layer can be computed by a uniform $\mathsf{TC}^0$ circuit with:
     \begin{itemize}
        \item
        Size: $\poly(n)$.
        \item
        Depth: $5d_\mathrm{std} + 2d_{\oplus} + d_\mathrm{sqrt}$.
     \end{itemize}
\end{lemma}

\subsection{Computing Flow Matching Layer in \texorpdfstring{$\TC^0$}{}.}\label{sec:fl_tc0}
We prove that the flow-matching layer can be efficiently simulated by a uniform $\mathsf{TC}^0$ circuit.
\begin{lemma}[Flow matching layer computation in $\TC^0$]\label{lem:fm_tc0}
     Given an input tensor $\X \in \R^{h \times w \times c}$. Let $\mathsf{NN}(X)$ denote the flow-matching layer defined in Definition~\ref{def:flow_matching_architecture}. Under the following constraints:
         {\bf Part 1.}
         Satisfy $h = w = n$,
         {\bf Part 2.}
         Channel bound: $c \leq n$,
         {\bf Part 3.}
         Precision bound: $p \leq \poly(n)$.
     
     The $\mathsf{NN}(\cdot, \cdot,\cdot)$ can be computed by a uniform $\mathsf{TC}^0$ circuit with
     \begin{itemize}
        \item 
        Size: $\poly(n)$.
        \item 
        Depth: $26d_{\mathrm{std}}+ 12 d_{\oplus} + 2d_{\mathrm{sqrt}} + d_{\exp}$.
    \end{itemize} 
    with $d_{\mathrm{std}}$ and $d_{\oplus}$ defined in Definition~\ref{lem:float_operations_TC}, $d_{\exp}$ defined in Definition~\ref{lem:exp} and $d_{\mathrm{sqrt}}$ defined in Definition~\ref{lem:sqrt}.
\end{lemma}

\begin{proof}
    {\bf Considering Step 1 in the flow-matching layer:}
    By Lemma~\ref{lem:mlp_tc0_formal}, parameters $\alpha_1$, $\alpha_2$,$\beta_1$,$\beta_2$,$\gamma_1$,$\gamma_2$ are computed via a circuit with:
        {\bf Part 1.}
        Depth: $2d_{\mathrm{std}} + d_\oplus$.
        {\bf Part 2.}
        Size: $\poly(n)$

    {\bf Considering Step 2 in flow-matching layer:} 
     By Lemma~\ref{lem:ln_tc0_formal}, $\mathsf{LN}(\F_i^t)$ requires depth $5d_\mathrm{std} + 2d_{\oplus} + d_\mathrm{sqrt}$. By Lemma~\ref{lem:float_operations_TC}, $A_1 = \gamma_1 \circ \mathsf{LN}(\F_t)+\beta_1$ requires depth $2d_{\mathrm{std}}$. By Lemma~\ref{lem:attn_tc0_formal}, $A_2 = \mathsf{Attn}(A_1)$ requires depth $6(d_{\mathrm{std}}+d_\oplus)+d_{\exp}$. By Lemma~\ref{lem:float_operations_TC} again, scaling $A_2 \circ \alpha_1$ requires depth $d_{\mathrm{std}}$. The total depth requires $14d_{\mathrm{std}} + 8d_{\oplus}+d_{\mathrm{sqrt}}+d_{\exp}$ for step 2.

    {\bf Considering Step 3 in flow-matching layer:} By Lemma~\ref{lem:ln_tc0_formal}, $\mathsf{LN}(\F'^t_i)$ requires depth $5d_\mathrm{std} + 2d_{\oplus} + d_\mathrm{sqrt}$. By Lemma~\ref{lem:float_operations_TC}, $A_3 = \gamma_2 \circ \mathsf{LN}(\wh{\F}_t)+\beta_2$ requries depth  $2d_{\mathrm{std}}$. By Lemma~\ref{lem:mlp_tc0_formal}, $A_4 = \mathsf{MLP}(A_3,c,c)$ requires depth $2d_{\mathrm{std}} + d_\oplus$.
    By Lemma~\ref{lem:float_operations_TC} again, $A_4 \circ \alpha_2$ requires depth $d_{\mathrm{std}}$. The total depth requires $10d_{\mathrm{std}}+3d_{\oplus}+d_{\mathrm{sqrt}}$  for step 3..

    Finally, combining the result above, we need a circuit with depth $26d_{\mathrm{std}}+ 12 d_{\oplus} + 2d_{\mathrm{sqrt}} + d_{\exp}$ and size $\poly(n)$ to simulate the flow-matching layer.
    
\end{proof}

\subsection{Circuit Complexity Bound for FlowAR Architecture}\label{sec:main_result_flowar}
We present that the FlowAR Model can be efficiently simulated by a uniform $\mathsf{TC}^0$ circuit.
\begin{theorem}[FlowAR Model computation in $\TC^0$]\label{thm:flowar_tc0}
Given an input tensor $\X \in \R^{h \times w \times c}$. Under the following constraints:
    {\bf Part 1.}
    Satisfy $h = w = n$,
    {\bf Part 2.}
    Channel bound: $c \leq n$,
    {\bf Part 3.}
    Precision bound: $p \leq \poly(n)$.
    {\bf Part 4.}
    Number of scales: $K = O(1)$,
    {\bf Part 5.}
    $d_{\mathrm{std}},d_\oplus,d_{\mathrm{sqrt}},d_{\exp} = O(1)$.

Then, the FlowAR Model can be simulated by a uniform $\TC^0$ circuit family.

\end{theorem}
\begin{proof}
    For every $i \in [K]$, by Lemma~\ref{lem:up_tc0}, Lemma~\ref{lem:down_tc0},  Lemma~\ref{lem:attn_tc0_informal}, Lemma~\ref{lem:ffn_tc0_informal} and Lemma~\ref{lem:fm_tc0}, we can simulate the $i$-th layer of FlowAR Model with a uniform $\TC^0$ circuit whose size is $\poly(n)$ and depth is $O(1)$. Since the total number of layers $K = O(1)$, the composition of all $K$ circuits yields a single uniform $\mathsf{TC}^0$ circuit with:
        {\bf Part 1.}
        Size: $\poly(n)$.
        {\bf Part 2.}
        Depth: $O(1)$.
\end{proof}

In Theorem~\ref{thm:flowar_tc0}, we establish that a FlowAR model with $\poly(n)$ precision, constant depth, and $\poly(n)$ size can be efficiently simulated by a $\mathsf{DLOGTIME}$-uniform $\TC^0$ circuit family. This indicates that while the flow-matching architecture enhances the capability of visual autoregressive models, the FlowAR architecture remains inherently limited in expressivity under circuit complexity theory.
\section{Provably Efficient Criteria}\label{sec:efficient_critieria}

\subsection{Approximate Attention Computation}\label{sec:fast_attn}
In this section, we introduce approximate attention computation, which can accelerate the computation of the attention layer.

\begin{definition}[Approximate Attention Computation $\mathsf{AAttC}(n, d, R, \delta)$, Definition 1.2 in \cite{as23}]\label{def:aattc}
    Given an input sequence $X \in \R^{n \times d}$ and an approximation tolerance $\delta > 0$. Let $Q,K,V \in \R^{n \times d}$ be weigh matrices bounded such that
    $
        \max\{\|Q\|_\infty,\|K\|_\infty,\|V\|_\infty\} \leq R.
    $
    The {\bf Approximate Attention Computation} $\mathsf{AAttC}(n, d, R, \delta)$ outputs a matrix $N \in \R^{n \times d }$ satisfying
    $
        \| N - \mathsf{Attn}(X)\|_\infty \leq \delta
    $
\end{definition}

Next, we present a lemma that demonstrates the computational time cost of the AATTC method.
\begin{lemma}[Fast Attention via Subquadratic Computation, Theorem 1.4 of \cite{as23}]\label{lem:as23_attention}
Let $\mathsf{AAttC}$ be formalized as in Definition~\ref{def:aattc}. 
For parameter configurations:
    Part 1.
    Embedding dimension $d = O(\log n)$,
    Part 2.
    $R = \Theta(\sqrt{\log n})$,
    Part 3.
    Approximation tolerance $\delta = 1/\poly(n)$,
the $\mathsf{AAttC}$ computation satisfies
$
    \mathcal{T}(n, n^{o(1)}, d) = n^{1 + o(1)},
$
where $\mathcal{T}$ denotes the time complexity under these constraints.

\end{lemma}

\subsection{Fast FlowAR Architecture in the Inference Pipleline}\label{sec:fast_flowar}
Firstly, we define the fast flow-matching layer, where the $\mathsf{Attn}$ layers in the original flow-matching module are replaced with $\mathsf{AAttC}$ layers.

\begin{definition}[Fast Flow Matching Architecture]\label{def:fast_flow_matching_architecture}
Given the following:
\begin{itemize}
    \item {\bf Input tensor:} $\X \in \R^{h \times w \times c}$.
    \item {\bf Scales number:} $K \in \mathbb{N}$.
    \item {\bf Token maps:} For $i \in [K]$, $\wh{\Y}_i \in \R^{(h / r_i) \times (w/r_i) \times c}$ denote the token maps generated by autoregressive transformer defined in Definition~\ref{def:ar_transformer}.
    \item {\bf Interpolation Tokens:} For $i \in [K]$, $\F_i^t \in \R^{(h / r_i) \times (w/r_i) \times c}$ denote interpolated input defined in Definition~\ref{def:flow}.
    \item {\bf Time step:} For $i \in [K]$, $t_i \in [0,1]$ denotes timestep.
    \item {\bf Approximate Attention layer:}  For $i \in [K]$, $\mathsf{AAttC}_i(\cdot):\R^{h/r_i \times w/r_i \times c} \to \R^{h/r_i \times w/r_i \times c}$ is defined in Definition~\ref{def:attn_layer}.
    \item {\bf MLP layer:}  For $i \in [K]$, $\mathsf{MLP}_i(\cdot,c,d):\R^{h/r_i \times w/r_i \times c} \to \R^{h/r_i \times w/r_i \times c}$ is defined in Definition~\ref{def:mlp}.
    \item {\bf LN layer:} For $i \in [K]$, $\mathsf{LN}_i(\cdot):\R^{h/r_i \times w/r_i \times c} \to \R^{h/r_i \times w/r_i \times c}$ is defined in Definition~\ref{def:ln}.
\end{itemize}
The computation steps of flow-matching layers are as follows:
\begin{itemize}
    \item 
    {\bf Time-conditioned parameter generation:}
    \begin{align*}
        &~\alpha_1, \alpha_2, \beta_1, \beta_2, \gamma_1, \gamma_2\\:=&~  \mathsf{MLP}_i(\wh{\Y}_i + t_i \cdot {\bf 1}_{(h / r_i) \times (w/r_i) \times c},c,6c)
    \end{align*}
    \item 
    {\bf Intermediate variable computation:}
    \begin{align*}
        \F'^{t_i}_i:= \mathsf{AAttC}_i (\gamma_1 \circ \mathsf{LN}(\F_i^{t_i}) + \beta_1) \circ \alpha_1
    \end{align*}
    with $\circ$ denoting Hadamard (element-wise) product.
    \item 
    {\bf Final projection:}
    \begin{align*}
        \F''^{t_i}_i := \mathsf{MLP}_i(\gamma_2 \circ \mathsf{LN}(\F'^{t_i}_i)+ \beta_2,c,c) \circ \alpha_2
    \end{align*}
\end{itemize}

The operation is denoted as $\F''^{t_i}_i := \mathsf{FNN}_i(\wh{\Y_i},\F_i^{t_i},t_i)$
\end{definition}

Next, we define the Fast FlowAR inference pipeline architecture, where all 
$\mathsf{Attn}$ layers in the original FlowAR architecture are replaced with $\mathsf{AAttC}$ layers.

\begin{definition}[Fast FlowAR Inference Architecture]\label{def:fast_flow_architecture_inference}
    Given the following:
    \begin{itemize}
        \item 
        {\bf Scales number:} $K \in \mathbb{N}$.
        \item 
        {\bf Scale factor:} For $i \in [K]$, $r_i:= a^{K-i}$ where base factor $a \in \mathbb{N}^+$.
         \item 
        {\bf Upsampling functions:}  For $i \in [K]$, $\phi_{\mathrm{up},i}(\cdot,a): \R^{(h/r_i) \times (w/r_i) \times c}\to \R^{(h/r_{i+1}) \times (w/r_{i+1}) \times c}$ from Definition~\ref{def:bicubic_up_sample_function}.
        \item 
        {\bf Approximate Attention layer:}  For $i \in [K]$, $\mathsf{AAttC}_i(\cdot):\R^{(\sum_{j=1}^i h/r_j \cdot w/r_{j})\times c} \to \R^{(\sum_{j=1}^i h/r_j \cdot w/r_{j})\times c}$ which acts on flattened sequences of dimension defined in Definition~\ref{def:aattc}.
        \item 
        {\bf Feed forward layer: } For $i \in [K]$, $\mathsf{FFN}_i(\cdot): \R^{(\sum_{j=1}^i h/r_j \cdot w/r_{j})\times c} \to \R^{(\sum_{j=1}^i h/r_j \cdot w/r_{j})\times c}$ which acts on flattened sequences of dimension defined in Definition~\ref{def:ffn}.
        \item 
        {\bf Fast flow-matching layer:} For $i \in [K]$, $\mathsf{FNN}_i(\cdot,\cdot,\cdot):\R^{h/r_i \times w/r_i \times c}\times \R^{h/r_i \times w/r_i \times c}\times \R \to \R^{h/r_i \times w/r_i \times c}$ denote the fast flow-matching layer defined in Definition~\ref{def:fast_flow_matching_architecture}.
        \item 
        {\bf Initial condition:} $\Z_{\mathrm{init}} \in \R^{(h/r_1) \times (w/r_1) \times c}$ denotes the initial token maps which encodes class information.
        \item 
        {\bf Time steps:} For $i \in [K]$, $t_i \in [0,1]$ denotes time steps.
        \item 
        {\bf Interpolated inputs:} For $i \in [K]$, $\F_i^{t_i} \in \R^{h/r_i \times w/r_i \times c}$ defined in Definition~\ref{def:flow}.
        \item 
        {\bf Cumulative dimensions:} We define $\wt{h}_i := \sum_{j=1}^i h/r_j$ and  $\wt{w}_i := \sum_{j=1}^i w/r_j$ for $i \in [K]$.
    \end{itemize}
    
    The FlowAR model conducts the following recursive construction:
    
    \begin{itemize}
        \item 
        {\bf Base case $i=1$:}
        \begin{align*}
            \Z_1 =&~ \Z_{\mathrm{init}},\\
            \wh{\Y}_1 =&~ \mathsf{FFN}_1(\mathsf{AAttC}_1(\Z_1)),\\
            \wt{\Y}_1 =&~ \mathsf{FNN}_1(\wh{\Y}_1,\F_{1}^{t_1},t_1).
        \end{align*}
        
        \item 
        {\bf Inductive step $i \geq 2$:}
        \begin{itemize}
            \item 
            Spatial aggregation:
            \begin{align*}
                &~\Z_i \\=&~ \mathsf{Concat}(\Z_{\mathrm{init}},\phi_{\mathrm{up},1}(\wt{\Y}_{i-1}),\dots,\phi_{\mathrm{up},i-1}(\wt{\Y}_{i-1})) 
            \end{align*}
            \item 
            Autoregressive transformer computation:
            \begin{align*}
                \wh{\Y}_i = \mathsf{FFN}_i(\mathsf{AAttC}_i(\Z_1))_{\wt{h}_{i-1}:\wt{h}_{i-1},\wt{w}_{i}:\wt{w}_{i},0:c}
            \end{align*}
            \item 
            Flow matching layer:
            \begin{align*}
                \wt{\Y}_i = \mathsf{FNN}_i(\wh{\Y}_i,\F_{i}^{t_i},t_i)
            \end{align*}
        \end{itemize}
        
        The final output is $\wt{\Y}_K \in \R^{h \times w \times c}$.
    \end{itemize}
   
\end{definition}

\subsection{Running Time}\label{sec:running_time}
In this section, we analyzed the running time required by the original FlowAR architecture and the running time required by the Fast FlowAR architecture. The results indicate that by adopting the Approximate Attention computation module, we can accelerate the running time of FlowAR to almost quadratic time.

First, we present the results of the running time analysis for the original FlowAR model.
\begin{lemma}[Inference Runtime of Original FlowAR Architecture, informal version of Lemma~\ref{lem:runtime_old_flowar_formal}]\label{lem:runtime_old_flowar_informal}
    Consider the original FlowAR inference pipeline with the following parameters:
        {\bf Input tensor:} $\X \in \R^{h \times w \times c}$. Assume $h=w=n$ and $c = O(\log n)$.
        {\bf Number of scales:} $K = O(1)$.
        {\bf Scale factor:} For $i \in [K]$, $r_i:= a^{K-i}$ where base factor $a \in \mathbb{N}^+$.
        {\bf Upsampling functions}  For $i \in [K]$, $\phi_{\mathrm{up},i}(\cdot,a)$ from Definition~\ref{def:bicubic_up_sample_function}.
        {\bf Attention layer:}  For $i \in [K]$, $\mathsf{Attn}_i(\cdot)$ which acts on flattened sequences of dimension defined in Definition~\ref{def:attn_layer}.
        {\bf Feed forward layer: } For $i \in [K]$, $\mathsf{FFN}_i(\cdot)$ which acts on flattened sequences of dimension defined in Definition~\ref{def:ffn}.
        {\bf Flow matching layer:} For $i \in [K]$, $\mathsf{NN}_i(\cdot,\cdot,\cdot)$ denote the flow-matching layer defined in Definition~\ref{def:flow_matching_architecture}.
    
    Under these conditions, the total inference runtime of FlowAR is bounded by 
    \begin{align*}
    O(n^{4+o(1)}).
    \end{align*}
\end{lemma}

Then, we present the results of the running time analysis for the fast FlowAR model.
\begin{lemma}[Inference Runtime of Fast FlowAR Architecture, informal version of Lemma~\ref{lem:runtime_fast_flowar_formal}]\label{lem:runtime_fast_flowar_informal}
    Consider the original FlowAR inference pipeline with the following parameters:
        {\bf Input tensor:} $\X \in \R^{h \times w \times c}$. Assume $h=w=n$ and $c = O(\log n)$.
        {\bf Number of scales:} $K = O(1)$.
        {\bf Scale factor:} For $i \in [K]$, $r_i:= a^{K-i}$ where base factor $a \in \mathbb{N}^+$.
        {\bf Upsampling functions}  For $i \in [K]$, $\phi_{\mathrm{up},i}(\cdot,a)$ from Definition~\ref{def:bicubic_up_sample_function}.
        {\bf Approximate Attention layer:}  For $i \in [K]$, $\mathsf{AAttC}_i(\cdot)$ defined in Definition~\ref{def:aattc}.
        {\bf Feed forward layer: } For $i \in [K]$, $\mathsf{FFN}_i(\cdot)$ which acts on flattened sequences of dimension defined in Definition~\ref{def:ffn}.
        {\bf Fast flow-matching layer:} For $i \in [K]$, $\mathsf{FNN}_i(\cdot,\cdot,\cdot)$ denote the fast flow-matching layer defined in Definition~\ref{def:fast_flow_matching_architecture}.
    
    Under these conditions, the total inference runtime of FlowAR is bounded by 
    \begin{align*}
    O(n^{2+o(1)}).
    \end{align*}
\end{lemma}

\subsection{Error Propagation Analysis}\label{lem:error_propagation_analysis}
In this section, we present an error analysis introduced by the fast algorithm applied to the FlowAR model.
\begin{lemma}[Error Bound Between Fast FlowAR and FlowAR Outputs, informal version of Lemma~\ref{lem:error_analysis_fast_flowar}]\label{lem:error_analysis_fast_flowar_informal}

    Under certain conditions, the $\ell_\infty$ error between the final outputs is bounded by
    \begin{align*}
        \|\wt{\Y}'_K - \wt{\Y}_K\|_\infty \leq 1/\poly(n).
    \end{align*}
\end{lemma}

\subsection{Existence of Almost Quadratic Time Algorithm}\label{sec:almost_quadratic_time_algorithm}
This section presents a theorem proving the existence of a quadratic-time algorithm that speeds up the FlowAR architecture and guarantees a bounded additive error.
\begin{theorem}[Existence of Almost Quadratic Time Algorithm]
\label{thm:upper_bound:formal}
Suppose $d = O(\log n)$ and $R = o(\sqrt{\log n})$. There is an algorithm that approximates the  FlowAR architecture up to $1/\poly(n)$ additive error in $O(n^{2+o(1)})$ time.
\end{theorem}
\begin{proof}
    By combining the result of Lemma~\ref{lem:runtime_fast_flowar_informal} and Lemma~\ref{lem:error_analysis_fast_flowar_informal}, we can easily derive the proof.
\end{proof}

Our Theorem~\ref{thm:upper_bound:formal} shows that we can accelerate FlowAR while only introducing a small error. Using the low-rank approximation in the attention mechanism is also used in previous works \cite{kll+25,lls+24_conv,llss25,lss+25_relu,chl+24_rope,lss+24,lssz24_tat,as24_iclr,as24b,as24_rope,hsk+24}.
\section{Conclusion}\label{sec:conclusion}
In this work, we have addressed several fundamental questions about the FlowAR architecture, making significant contributions to both theoretical understanding and complexity efficiency. By rigorously analyzing the architecture of FlowAR, we demonstrated that despite its sophisticated integration of flow-based and autoregressive mechanisms, it resides within the complexity class $\TC^0$. Specifically, we proved that each module of FlowAR, including the attention and flow-matching layers, can be simulated by a constant-depth, polynomial-size circuit. 

Beyond the circuit theoretical analysis, we identified the computational bottleneck in FlowAR's attention mechanism and developed an efficient variant using low-rank approximation techniques. This optimization achieves nearly quadratic runtime $O(n^{2+o(1)})$, a substantial improvement over the original $O(n^{4+o(1)})$ complexity, while maintaining an error bound of $1/\poly(n)$. 

Our findings provide both a theoretical foundation for understanding the computational limits of FlowAR and practical guidelines for implementing more efficient variants, offering valuable insights for future development of generative architectures and establishing a framework for comparisons with other generative paradigms.

\ifdefined\isarxiv
\bibliographystyle{alpha}
\bibliography{ref}
\else
\bibliographystyle{apalike}
\bibliography{ref}
\fi

\newpage
\clearpage
\section*{Checklist}

\begin{enumerate}

  \item For all models and algorithms presented, check if you include:
  \begin{enumerate}
    \item A clear description of the mathematical setting, assumptions, algorithm, and/or model. 
    [Yes]
    \item An analysis of the properties and complexity (time, space, sample size) of any algorithm. [Yes] 
    \item (Optional) Anonymized source code, with specification of all dependencies, including external libraries. [Not Applicable] 
  \end{enumerate}

  \item For any theoretical claim, check if you include:
  \begin{enumerate}
    \item Statements of the full set of assumptions of all theoretical results. [Yes] 
    \item Complete proofs of all theoretical results. [Yes] 
    \item Clear explanations of any assumptions. [Yes] 
  \end{enumerate}

  \item For all figures and tables that present empirical results, check if you include:
  \begin{enumerate}
    \item The code, data, and instructions needed to reproduce the main experimental results (either in the supplemental material or as a URL). [Not Applicable]
    \item All the training details (e.g., data splits, hyperparameters, how they were chosen). [Not Applicable]
    \item A clear definition of the specific measure or statistics and error bars (e.g., with respect to the random seed after running experiments multiple times). [Not Applicable]
    \item A description of the computing infrastructure used. (e.g., type of GPUs, internal cluster, or cloud provider). [Not Applicable]
  \end{enumerate}

  \item If you are using existing assets (e.g., code, data, models) or curating/releasing new assets, check if you include:
  \begin{enumerate}
    \item Citations of the creator If your work uses existing assets. [Not Applicable] 
    \item The license information of the assets, if applicable. [Not Applicable] 
    \item New assets either in the supplemental material or as a URL, if applicable. [Not Applicable] 
    \item Information about consent from data providers/curators. [Not Applicable]
    \item Discussion of sensible content if applicable, e.g., personally identifiable information or offensive content. [Not Applicable]
  \end{enumerate}

  \item If you used crowdsourcing or conducted research with human subjects, check if you include:
  \begin{enumerate}
    \item The full text of instructions given to participants and screenshots. [Not Applicable] 
    \item Descriptions of potential participant risks, with links to Institutional Review Board (IRB) approvals if applicable. [Not Applicable] 
    \item The estimated hourly wage paid to participants and the total amount spent on participant compensation. [Not Applicable]
  \end{enumerate}

\end{enumerate}


\clearpage
\appendix
\thispagestyle{empty}
\onecolumn

\aistatstitle{\paperTitle \\ Supplementary Materials}

{\bf Roadmap.}
Section~\ref{sec:app_notations} presents all the notations of this paper. 
In Section~\ref{sec:model_formulation_of_flowar} presents the formal definition of every module of FlowAR.
In Section~\ref{sec:app_missing_proof}, we present some missing proofs in Section~\ref{sec:main_result}. 
Section~\ref{sec:app_efficient_critieria} presents provably efficient criteria of the fast FlowAR model.

\section{Notations}\label{sec:app_notations}
Given a matrix $X \in \R^{hw \times d}$, we denote its tensorized form as $\X \in \R^{h \times w \times d}$. Additionally, we define the set $[n]$ to represent $\{1,2,\cdots, n\}$ for any positive integer $n$. We define the set of natural numbers as $\mathbb{N}:= \{0,1,2,\dots\}$. Let $X \in \mathbb{R}^{m \times n}$ be a matrix, where $X_{i,j}$ refers to the element at the $i$-th row and $j$-th column. When $x_i$ belongs to $\{ 0,1 \}^*$, it signifies a binary number with arbitrary length. In a general setting, $x_i$ represents a length $p$ binary string, with each bit taking a value of either 1 or 0. Given a matrix $X \in \R^{n \times d}$, we define $\|X\|_\infty  $ as the maximum norm of $X$. Specifically, $\|X\|_\infty = \max_{i,j} |X_{i,j}|$. 
\section{Model Formulation for FlowAR Architecture}\label{sec:model_formulation_of_flowar}
In this section, we provide a mathematical definition for every module of FlowAR. 
Section~\ref{sec:sample_function} provides the definition of up-sample and down-sample functions.
In Section~\ref{sec:downsample_tokenizer}, we mathematically define the VAE tokenizer.  Section~\ref{sec:ar_transformer} presents a mathematical formulation for every module in the autoregressive transformer in FlowAR. Section~\ref{sec:flow_matching} provides some important definitions of the flow-matching architecture. In Section~\ref{sec:inference_of_flowar}, we also provide a rigorous mathematical definition for the overall architecture of the FlowAR Model during the inference process.

\subsection{Sample Function}\label{sec:sample_function}
We define the bicubic upsampling function.

\begin{definition}[Bicubic Upsampling Function]\label{def:bicubic_up_sample_function}
Given the following:
\begin{itemize}
    \item {\bf Input tensor:} $\X \in \R^{h \times w \times c}$ where $h,w,c$ represent height, width, and the number of channels, respectively.
    \item {\bf Scaling factor:} A positive integer $r \geq 1$.
    \item {\bf Bicubic kernel:} $W:\R \to [0,1]$
\end{itemize}
The bicubic upsampling function $\phi_{\mathrm{up}}(\X,r)$ computes an output tensor $\Y \in \R^{rh \times rw \times c}$. For every output position $i \in [rh], j \in [rw], l \in [c]$:
\begin{align*}
    \Y_{i,j,l} =  \sum_{s=-1}^2 \sum_{t=-1}^2 W(s) \cdot W(t) \cdot \X_{\lfloor \frac{i}{r}\rfloor+s, \lfloor \frac{j}{r}\rfloor+t,l}
\end{align*}

\end{definition}
Next, we define the downsampling function.
\begin{definition}[Linear Downsampling Function]\label{def:linear_down_sample_function}
Given the following:
\begin{itemize}
    \item {\bf Input tensor:} $\X \in \R^{h \times w \times c}$ where $h,w,c$ represent height, width, and the number of channels, respectively.
    \item {\bf Scaling factor:} A positive integer $r \geq 1$.
\end{itemize}
The linear downsampling function $\phi_{\mathrm{down}}(\X,r)$ computes an output tensor $\Y \in \R^{(h/r) \times (w/r) \times c}$. Let $\Phi_{\mathrm{down}} \in \R^{(h/r \cdot w/r) \times hw}$ denote a linear transformation matrix. Reshape $\X$ into the matrix $X \in \R^{hw \times c}$ by flattening its spatial dimensions.  The output matrix is defined via:
\begin{align*}
     Y = \Phi_{\mathrm{down}}X \in \R^{(h/r \cdot w/r) \times c},
\end{align*}
Then reshaped back to $\Y \in \R^{(h/r) \times(w/r) \times c}$.
\end{definition}

\subsection{Multi-Scale Downsampling Tokenizer}\label{sec:downsample_tokenizer}
Given an input image, the FlowAR model will utilize the VAE to generate latent representation $\X \in {\R^{h \times w \times c}}$. To meet the requirements of Multi-Scale autoregressive image generation, FlowAR uses a Multi-Scale VAE Tokenizer to downsample $\X$ and generate Token Maps of different sizes.
\begin{definition}[Multi-Scale Downsampling Tokenizer]\label{def:downsample_tokenizer}
 Given the following:
 \begin{itemize}
     \item {\bf Latent representation tensor:} $\X\in \R^{h \times w \times c}$ generated by VAE.
     \item {\bf Number of scales:} $K \in \mathbb{N}$.
     \item {\bf Base scaling factor:}  positive integer $a \geq 1$
     \item {\bf Downsampling functions:} For $i \in [K]$, define scale-specific factors $r_i := a^{K-i}$ and use the linear downsampling function $\phi_{\mathrm{down}}(\X,r_i)$ from Definition~\ref{def:linear_down_sample_function}.
 \end{itemize}
 Then tokenizer outputs a sequence of token maps $\{\Y^2, \Y^2,\dots, \Y^K\}$, where the $i$-th token map is
 \begin{align*}
     \Y^i := \phi_{\mathrm{down},i}(\X,r_i) \in \R^{(h / r_i) \times (w/r_i) \times c},
 \end{align*}
 Formally, the tokenizer is defined as
 \begin{align*}
     \mathsf{TN}(\X) := \{\Y^{1}, \dots,\Y^{K}\}.
 \end{align*} 
\end{definition}

\begin{remark}
    In \cite{ryh+24}, the base factor is set to $a = 2$, resulting in exponentially increasing scales $r_i = 2^{K-i}$ for $i \in [K]$.
\end{remark}

\subsection{Autoregressive Transformer}\label{sec:ar_transformer}

The autoregressive transformer is a key module of the FlowAR model. We will introduce each layer of autoregressive transformer in this section.
\begin{definition}[Attention Layer]\label{def:attn_layer}
Given the following:
\begin{itemize}
    \item {\bf Input tensor:} $\X \in \R^{h \times w \times c}$ where $h,w,c$ represent height, width, and the number of channels, respectively.
    \item {\bf Weight matrices:} $W_Q,W_K,W_V \in \R^{c \times c}$ will be used in query, key, and value projection, respectively.
\end{itemize}
The attention layer $\mathsf{Attn}(\X)$ computes an output tensor $\Y \in \R^{h \times w \times c}$ as follows:
\begin{itemize}
    \item {\bf Reshape:} Flatten $\X$ into a matrix $X \in \R^{hw \times c}$ with spatial dimensions collapsed.
    \item {\bf Compute attention matrix:} For $i,j \in [hw]$, compute pairwise scores:
    \begin{align*}
        A_{i,j} := & ~\exp(  X_{i,*}   W_Q   W_K^\top   X_{j,*}^\top), \text{~~for~} i, j \in [hw].
    \end{align*}
    \item {\bf Normalization:} Compute diagnal matrix $D:=\diag(A {\bf 1}_n) \in \R^{hw \times hw}$, where ${\bf 1}_n$ is the all-ones vector. And compute:
    \begin{align*}
         Y := D^{-1}AXW_V \in \R^{hw \times c}.
    \end{align*}
    then reshape $Y$ to $\Y \in \R^{h \times w \times c}$.
\end{itemize}
\end{definition}

Then, we define the multiple-layer perception layer.
\begin{definition}[MLP layer]\label{def:mlp}
Given the following:
\begin{itemize}
    \item {\bf Input tensor:} $\X \in \R^{h \times w \times c}$ where $h,w,c$ represent height, width, and the number of channels, respectively.
    \item {\bf Weight matrices and bias vector:} $W \in \R^{c \times d}$ and $b \in \R^{1 \times d}$.
\end{itemize}
The MLP layer computes an output tensor $\Y \in \R^{h \times w \times d}$ as follows:
\begin{itemize}
    \item {\bf Reshape:} Flatten $\X$ into a matrix $X \in \R^{hw \times c}$ with spatial dimensions collapsed.
    \item {\bf Affine transformation:} For all $j \in [hw]$, compute
    \begin{align*}
        Y_{j,*} = \underbrace{X_{j,*}}_{1\times c} \cdot \underbrace{W}_{c \times d} + \underbrace{b}_{1 \times d}
    \end{align*}
    Then reshape $Y \in \R^{hw \times d}$ into $\Y \in \R^{h \times w \times d}$.
\end{itemize}
The operation is denoted as $\Y := \mathsf{MLP}(\X,c,d)$.
\end{definition}

Next, we introduce the definition of the feedforward layer.
\begin{definition}[Feed forward layer]\label{def:ffn}
Given the following:
\begin{itemize}
    \item {\bf Input tensor:} $\X \in \R^{h \times w \times c}$ where $h,w,c$ represent height, width, and the number of channels, respectively.
    \item {\bf Weight matrices and bias vector:} $W_1, W_2 \in \R^{c \times d}$ and $b_1, b_2 \in \R^{1 \times d}$.
    \item {\bf Activation:} $\sigma:\R \to \R$ denotes the $\mathsf{ReLU}$ activation function which is applied element-wise.
\end{itemize}
The feedforward layer computes an output tensor $\Y \in \R^{h \times w \times d}$ as follows:
\begin{itemize}
    \item {\bf Reshape:} Flatten $\X$ into a matrix $X \in \R^{hw \times c}$ with spatial dimensions collapsed.
    \item {\bf Transform:} For each $j \in [hw]$, compute 
    \begin{align*}
        Y_{j,*}=  \underbrace{X_{j,*}}_{1 \times c} +  \sigma (\underbrace{X_{j,*}}_{1\times c} \cdot \underbrace{W_1}_{c \times c} + \underbrace{b_1}_{1\times c}) \cdot \underbrace{W_2}_{c \times c} + \underbrace{b_2}_{1 \times c} \in \R^{1 \times c}
    \end{align*}
    where $\sigma$ acts element-wise on intermediate results. Then reshape $Y \in \R^{hw \times c}$ into $\Y \in \R^{h \times w \times c}$.
\end{itemize}
The operation is denoted as $\Y := \mathsf{FFN}(\X)$.

\end{definition}

To move on, we define the layer normalization layer.
\begin{definition}[Layer Normalization Layer]\label{def:ln}
    Given the following:
    \begin{itemize}
        \item {\bf Input tensor:} $\X \in \R^{h \times w \times c}$ where $h,w,c$ represent height, width, and the number of channels, respectively.
    \end{itemize}
    The layer normalization computes $\Y$ through
    \begin{itemize}
        \item {\bf Reshape:} Flatten $\X$ into a matrix $X \in \R^{hw \times c}$ with spatial dimensions collapsed.
        \item {\bf Normalize:} For each $j \in [hw]$, compute
        \begin{align*}
            Y_{j,*} =  \frac{X_{j,*}-\mu_j}{\sqrt{\sigma_j^2}}
        \end{align*}
        where
        \begin{align*}
            \mu_j := \sum_{k=1}^c X_{j,k}/c, ~~ \sigma_{j}^2 = \sum_{k=1}^c(X_{j,k}-\mu_j)^2/c
        \end{align*}
        Then reshape $Y \in \R^{hw \times c}$ into $\Y \in \R^{h \times w \times c}$.
    \end{itemize}
    The operation is denoted as $\Y := \mathsf{LN}(\X)$.
\end{definition}

Now, we can proceed to show the definition of the autoregressive transformer.
\begin{definition}[Autoregressive Transformer]\label{def:ar_transformer}
    Given the following:
    \begin{itemize}
        \item {\bf Input tensor:} $\X \in \R^{h \times w \times c}$ where $h,w,c$ represent height, width, and the number of channels, respectively.
        \item {\bf Scales number:} $K \in \mathbb{N}$ denote the number of total scales in FlowAR.
        \item {\bf Token maps:} For $i \in [K]$, $\Y_i \in \R^{(h/r_i) \times (w/r_i) \times c}$ generated by the Multi-Scale Downsampling Tokenizer defined in Definition~\ref{def:downsample_tokenizer} where $r_i = a^{K-i}$ with base $a \in \mathbb{N}^+$.
        \item {\bf Upsampling functions:}  For $i \in [K]$, $\phi_{\mathrm{up},i}(\cdot,a): \R^{(h/r_i) \times (w/r_i) \times c}\to \R^{(h/r_{i+1}) \times (w/r_{i+1}) \times c}$ from Definition~\ref{def:bicubic_up_sample_function}.
        \item {\bf Attention layer:}  For $i \in [K]$, $\mathsf{Attn}_i(\cdot):\R^{(\sum_{j=1}^i h/r_j \cdot w/r_{j})\times c} \to \R^{(\sum_{j=1}^i h/r_j \cdot w/r_{j})\times c}$ which acts on flattened sequences of dimension defined in Definition~\ref{def:attn_layer}.
        \item {\bf Feed forward layer: } For $i \in [K]$, $\mathsf{FFN}_i(\cdot): \R^{(\sum_{j=1}^i h/r_j \cdot w/r_{j})\times c} \to \R^{(\sum_{j=1}^i h/r_j \cdot w/r_{j})\times c}$ which acts on flattened sequences of dimension defined in Definition~\ref{def:ffn}.
        \item {\bf Initial condition:} $\Z_{\mathrm{init}} \in \R^{(h/r_1) \times (w/r_1) \times c}$ denotes the initial token maps which encodes class information.
    \end{itemize}
    Then, the autoregressive processing is:
    \begin{enumerate}
        \item {\bf Initialization: } Let $\Z_1:=\Z_{\mathrm{init}}$.
        \item {\bf Iterative sequence construction:} For $i \geq 2$.
        \begin{align*}
            \Z_i := \mathsf{Concat}(\mathsf{Z}_{\mathrm{init}}, \phi_{\mathrm{up}, 1}(\Y^1, a), \ldots, \phi_{\mathrm{up}, i-1}(\Y^{i-1}, a)) \in \R^{(\sum_{j=1}^i h/r_j \cdot w/r_{j})\times c}
        \end{align*}
        where $\mathsf{Concat}$ reshapes tokens into a unified spatial grid.
        \item {\bf Transformer block:} For $i \in [K]$,
        \begin{align*}
            \mathsf{TF}_i(\Z_i) := \mathsf{FFN_i}(\mathsf{Attn}_i(\Z_i)) \in \R^{(\sum_{j=1}^i h/r_j \cdot w/r_{j})\times c}
        \end{align*}
        \item {\bf Output decomposition:} Extract the last scale's dimension   from the reshaped $\mathsf{TF}_i(\Z_i)$ to generate $\wh{\Y}_i \in \R^{(h/r_i) \times (w/r_i) \times c}$.
    \end{enumerate}
\end{definition}

\subsection{Flow Matching}\label{sec:flow_matching}

We begin by outlining the concept of velocity flow in the flow-matching architecture.
\begin{definition}[Flow]\label{def:flow}
Given the following:
\begin{itemize}
    \item {\bf Input tensor:} $\X \in \R^{h \times w \times c}$ where $h,w,c$ represent height, width, and the number of channels, respectively.
    \item {\bf Scales number:} $K \in \mathbb{N}$.
    \item {\bf Noise tensor: } For $i \in [K]$, $\F_i^0 \in \R^{(h / r_i) \times (w/r_i) \times c}$ with every entry sampled from $\mathcal{N}(0,1)$.
    \item {\bf Token maps:} For $i \in [K]$, $\wh{\Y}_i \in \R^{(h / r_i) \times (w/r_i) \times c}$ denote the token maps generated by autoregressive transformer defined in Definition~\ref{def:ar_transformer}.
\end{itemize}
Then, the model does the following:
\begin{itemize}
    \item {\bf Interpolation:} For timestep $t \in [0,1]$ and scale $i$,
    \begin{align*}
        \F_i^t := t \wh{\Y}_i + (1-t) \F_i^0
    \end{align*}
    defining a linear trajectory between noise $\F_0^i$ and target tokens  $\wh{\Y}_i$.
    \item {\bf Velocity Field:} The time-derivative of the flow at scale $i$ is 
    \begin{align*}
        \V^t_i := \frac{\d \F^t_{i}}{\d t} = \wh{\Y_i} -\F^0_i.
    \end{align*}
    constant across $t$ due to linear interpolation.
\end{itemize}
\end{definition}

To move forward, we propose an approach to enhance the performance of the flow-matching layer by replacing linear interpolation with a Quadratic Bézier curve.

\begin{definition}[High Order Flow]\label{def:high_order_flow}
Given the following:
\begin{itemize}
    \item {\bf Input tensor:} $\X \in \R^{h \times w \times c}$ where $h,w,c$ represent height, width, and the number of channels, respectively.
    \item {\bf Scales number:} $K \in \mathbb{N}$.
    \item {\bf Noise tensor: } For $i \in [K]$, $\F_i^0 \in \R^{(h / r_i) \times (w/r_i) \times c}$ with every entry sampled from $\mathcal{N}(0,1)$.
    \item {\bf Token maps:} For $i \in [K]$, $\wh{\Y}_i \in \R^{(h / r_i) \times (w/r_i) \times c}$ denote the token maps generated by autoregressive transformer defined in Definition~\ref{def:ar_transformer}.
\end{itemize}
Then, the model does the following:
\begin{itemize}
    \item {\bf Interpolation:} For timestep $t \in [0,1]$ and scale $i$,
    \begin{align*}
        \F_i^t := (1-t)^2 \F_i^0 + 2t(1-t) \mathsf{C}_i + t^2 \wh{\Y}_i
    \end{align*}
    defining a quadratic Bézier curve as the interpolation path between the initial noise and the target data. To be noticed, we take $\mathsf{C} = \frac{\F_i^0+\wh{\Y}_i}{2}$ as a control point that governs the curvature of the trajectory. This formulation replaces the standard linear interpolation with a higher-order flow, enabling a smoother and more flexible transition from noise to data in the flow-matching framework.
    \item {\bf Velocity Field:} The time-derivative of the flow at scale $i$ is 
    \begin{align*}
        \V^t_i := &~\frac{\d \F^t_{i}}{\d t}\\ =&~ -2(1-t)\F_i^0 + 2(1-2t) \mathsf{C}_i + 2t \wh{\Y}_i 
    \end{align*}
    constant across $t$ due to linear interpolation.
\end{itemize}
\end{definition}

We are now able to define the flow-matching layer, which is integrated in the FlowAR model.
\begin{definition}[Flow Matching Architecture]\label{def:flow_matching_architecture}
Given the following:
\begin{itemize}
    \item {\bf Input tensor:} $\X \in \R^{h \times w \times c}$ where $h,w,c$ represent height, width, and the number of channels, respectively.
    \item {\bf Scales number:} $K \in \mathbb{N}$ denote the number of total scales in FlowAR.
    \item {\bf Token maps:} For $i \in [K]$, $\wh{\Y}_i \in \R^{(h / r_i) \times (w/r_i) \times c}$ denote the token maps generated by autoregressive transformer defined in Definition~\ref{def:ar_transformer}.
    \item {\bf Interpolation Tokens:} For $i \in [K]$, $\F_i^t \in \R^{(h / r_i) \times (w/r_i) \times c}$ denote interpolated input defined in Definition~\ref{def:flow}.
    \item {\bf Time step:} For $i \in [K]$, $t_i \in [0,1]$ denotes timestep.
    \item {\bf Attention layer:}  For $i \in [K]$, $\mathsf{Attn}_i(\cdot):\R^{h/r_i \times w/r_i \times c} \to \R^{h/r_i \times w/r_i \times c}$ is defined in Definition~\ref{def:attn_layer}.
    \item {\bf MLP layer:}  For $i \in [K]$, $\mathsf{MLP}_i(\cdot,c,d):\R^{h/r_i \times w/r_i \times c} \to \R^{h/r_i \times w/r_i \times c}$ is defined in Definition~\ref{def:mlp}.
    \item {\bf LN layer:} For $i \in [K]$, $\mathsf{LN}_i(\cdot):\R^{h/r_i \times w/r_i \times c} \to \R^{h/r_i \times w/r_i \times c}$ is defined in Definition~\ref{def:ln}.
\end{itemize}
The computation steps of flow-matching layers are as follows:
\begin{itemize}
    \item {\bf Time-conditioned parameter generation:}
    \begin{align*}
        \alpha_1, \alpha_2, \beta_1, \beta_2, \gamma_1, \gamma_2:=  \mathsf{MLP}_i(\wh{\Y}_i + t_i \cdot {\bf 1}_{(h / r_i) \times (w/r_i) \times c},c,6c)
    \end{align*}
    \item {\bf Intermediate variable computation:}
    \begin{align*}
        \F'^{t_i}_i:= \mathsf{Attn}_i (\gamma_1 \circ \mathsf{LN}(\F_i^{t_i}) + \beta_1) \circ \alpha_1
    \end{align*}
    with $\circ$ denoting Hadamard (element-wise) product.
    \item {\bf Final projection:}
    \begin{align*}
        \F''^{t_i}_i := \mathsf{MLP}_i(\gamma_2 \circ \mathsf{LN}(\F'^{t_i}_i)+ \beta_2,c,c) \circ \alpha_2
    \end{align*}
\end{itemize}
The operation is denoted as $\F''^{t_i}_i := \mathsf{NN}_i(\wh{\Y_i},\F_i^{t_i},t_i)$
\end{definition}

\subsection{Inference of FlowAR Architecture}\label{sec:inference_of_flowar}
The inference phase of the FlowAR model differs from the training phase. During inference, neither the VAE nor the Multi-Scale Downsampling layers are used. Instead, given an initial token map representing class embeddings, the model autoregressively generates token maps across scales.
\begin{definition}[FlowAR Inference Architecture]\label{def:flow_architecture_inference}
    Given the following:
    \begin{itemize}
        \item {\bf Scales number:} $K \in \mathbb{N}$ denote the number of total scales in FlowAR.
        \item {\bf Scale factor:} For $i \in [K]$, $r_i:= a^{K-i}$ where base factor $a \in \mathbb{N}^+$.
        \item {\bf Upsampling functions:}  For $i \in [K]$, $\phi_{\mathrm{up},i}(\cdot,a): \R^{(h/r_i) \times (w/r_i) \times c}\to \R^{(h/r_{i+1}) \times (w/r_{i+1}) \times c}$ from Definition~\ref{def:bicubic_up_sample_function}.
        \item {\bf Attention layer:}  For $i \in [K]$, $\mathsf{Attn}_i(\cdot):\R^{(\sum_{j=1}^i h/r_j \cdot w/r_{j})\times c} \to \R^{(\sum_{j=1}^i h/r_j \cdot w/r_{j})\times c}$ which acts on flattened sequences of dimension defined in Definition~\ref{def:attn_layer}.
        \item {\bf Feed forward layer: } For $i \in [K]$, $\mathsf{FFN}_i(\cdot): \R^{(\sum_{j=1}^i h/r_j \cdot w/r_{j})\times c} \to \R^{(\sum_{j=1}^i h/r_j \cdot w/r_{j})\times c}$ which acts on flattened sequences of dimension defined in Definition~\ref{def:ffn}.
        \item {\bf Flow matching layer:} For $i \in [K]$, $\mathsf{NN}_i(\cdot,\cdot,\cdot):\R^{h/r_i \times w/r_i \times c}\times \R^{h/r_i \times w/r_i \times c}\times \R \to \R^{h/r_i \times w/r_i \times c}$ denote the flow-matching layer defined in Definition~\ref{def:flow_matching_architecture}.
        \item {\bf Initial condition:} $\Z_{\mathrm{init}} \in \R^{(h/r_1) \times (w/r_1) \times c}$ denotes the initial token maps which encodes class information.
        \item {\bf Time steps:} For $i \in [K]$, $t_i \in [0,1]$ denotes time steps.
        \item {\bf Interpolated inputs:} For $i \in [K]$, $\F_i^{t_i} \in \R^{h/r_i \times w/r_i \times c}$ defined in Definition~\ref{def:flow}.
        \item {\bf Cumulative dimensions:} We define $\wt{h}_i := \sum_{j=1}^i h/r_j$ and  $\wt{w}_i := \sum_{j=1}^i w/r_j$ for $i \in [K]$.
    \end{itemize}
    The FlowAR model conducts the following recursive construction:
    \begin{itemize}
        \item {\bf Base case $i=1$:}
        \begin{align*}
            &~\Z_1 = \Z_{\mathrm{init}}\\
            &~\wh{\Y}_1 = \mathsf{FFN}_1(\mathsf{Attn}_1(\Z_1))\\
            &~\wt{\Y}_1 = \mathsf{NN}_1(\wh{\Y}_1,\F_{1}^{t_1},t_1)
        \end{align*}
        \item {\bf Inductive step $i \geq 2$:}
        \begin{itemize}
            \item {\bf Spatial aggregation:}
            \begin{align*}
                \Z_i = \mathsf{Concat}(\Z_{\mathrm{init}},\phi_{\mathrm{up},1}(\wt{\Y}_{i-1}),\dots,\phi_{\mathrm{up},i-1}(\wt{\Y}_{i-1})) \in \R^{(\sum_{j=1}^i h/r_j \cdot w/r_j)\times c}
            \end{align*}
            \item {\bf Autoregressive transformer computation:}
            \begin{align*}
                \wh{\Y}_i = \mathsf{FFN}_i(\mathsf{Attn}_i(\Z_1))_{\wt{h}_{i-1}:\wt{h}_{i-1},\wt{w}_{i}:\wt{w}_{i},0:c}
            \end{align*}
            \item {\bf Flow matching layer:}
            \begin{align*}
                \wt{\Y}_i = \mathsf{NN}_i(\wh{\Y}_i,\F_{i}^{t_i},t_i)
            \end{align*}
        \end{itemize}
        The final output is $\wt{\Y}_K \in \R^{h \times w \times c}$.
    \end{itemize}
   
\end{definition}

\section{Supplementary Proof for Section~\ref{sec:main_result}}\label{sec:app_missing_proof}
In this section, we present some missing proofs in Section~\ref{sec:main_result}. 

\subsection{Computing Multiple-layer Perceptron in \texorpdfstring{$\TC^0$}{} }\label{sec:app_mlp_tc0}
This section presents the detailed proof for Lemma~\ref{lem:mlp_tc0_informal}.

\begin{lemma}[MLP computation in $\mathsf{TC}^0$, formal version of Lemma~\ref{lem:mlp_tc0_informal}]\label{lem:mlp_tc0_formal}
    Given an input tensor $\X\in \R^{h \times w \times c}$. Let $\mathsf{MLP}(\X,c,d)$ be the MLP layer defined in Definition~\ref{def:mlp}. Under the following constraints:
    \begin{itemize}
        \item Satisfy $h = w = n$,
        \item Channel bounds: $c, d\leq n$,
        \item Precision: $p \leq \poly(n)$,
    \end{itemize}
    The $\mathsf{MLP}(\X,c,d)$ function can be computed by a uniform $\mathsf{TC}^0$ circuit with:
    \begin{itemize}
        \item Size: $\poly(n)$.
        \item Depth: $2d_\mathrm{std} + d_{\oplus}$.
    \end{itemize}
    with $d_{\mathrm{std}}$ and $d_{\oplus}$ defined in Definition~\ref{lem:float_operations_TC}.
\end{lemma}

\begin{proof}
    For each $j \in [hw]$, by Lemma~\ref{lem:matrix_multi}, 
    compute $X_{j,*} \cdot W$ requires depth $d_{\mathrm{std}} + d_{\oplus}$. By Part 1 of Lemma~\ref{lem:float_operations_TC}, compute $X_{j,*} \cdot W + b$ requires depth $d_{\mathrm{std}}$.  Since for all $j \in [hw]$, the computation $X_{j,*} \cdot W + b$ can be simulated in parallel. Hence the total depth remains $2d_\mathrm{std} + d_{\oplus}$ and size is $\poly(n)$.
\end{proof}
\subsection{Computing Feed Forward Layer in \texorpdfstring{$\TC^0$}{} }\label{sec:app_ffn_tc0}
This section presents the detailed proof for Lemma~\ref{lem:ffn_tc0_informal}.

\begin{lemma}[FFN computation in $\mathsf{TC}^0$, formal version of Lemma~\ref{lem:ffn_tc0_informal}]\label{lem:ffn_tc0_formal}
 Given an input tensor $\X\in \R^{h \times w \times c}$. Let $\mathsf{FFN}(X):\R^{h \times w \times c} \to \R^{h \times w \times c}$ as defined in Definition~\ref{def:ffn}. Under the following constraints:
 \begin{itemize}
     \item Satisfy $h = w = n$,
         \item Channel bound: $c \leq n$,
         \item Precision bound: $p \leq \poly(n)$.
 \end{itemize}
The $\mathsf{FFN}(\X)$ layer can be computed by a uniform $\mathsf{TC}^0$ circuit with:
\begin{itemize}
    \item Size: $\poly(n)$.
    \item Depth: $6d_\mathrm{std} + 2d_{\oplus}$.
\end{itemize}
with $d_{\mathrm{std}}$ and $d_{\oplus}$ defined in Definition~\ref{lem:float_operations_TC}.
\end{lemma}
\begin{proof}
    For each $j \in [hw]$, by the proof of Lemma~\ref{lem:mlp_tc0_formal}, compute $X_{j,*} \cdot W_1 + b_1$ requires depth $2d_{\mathrm{std}} + d_{\oplus}$. By Lemma~\ref{lem:float_operations_TC}, compute $A_1 = \sigma(X_{j,*} \cdot W + b)$ requires depth $d_{\mathrm{std}}$. By applying Lemma~\ref{lem:mlp_tc0_formal} again, compute $A_2 = A_1\cdot W_2 +b_2$ requires depth $2d_{\mathrm{std}} + d_{\oplus}$. Lastly, by Part 1 of Lemma~\ref{lem:float_operations_TC}, compute $X_{j,*} + A_2$ requires depth $d_{\mathrm{std}}$.
    
    Combing the result above, we can have that compute $Y_{j,*}=X_{j,*} + \sigma(X_{j,*} \cdot W_1 + b_1)\cdot W_2 +b_2$ requires depth $6d_{\mathrm{std}}+2d_{\oplus}$.

    Since for all $j \in [hw]$, the computation $Y_{j,*}$ can be simulated in parallel. Hence the total depth remains $6d_{\mathrm{std}}+2d_{\oplus}$ and size is $\poly(n)$.
\end{proof}

\subsection{Computing Attention Layer in \texorpdfstring{$\TC^0$}{} }\label{sec:app_attn_tc0}
This section presents the detailed proof for Lemma~\ref{lem:attn_tc0_informal}.

\begin{lemma}[Attention layer computation in $\mathsf{TC}^0$, formal version of Lemma~\ref{lem:attn_tc0_informal}]\label{lem:attn_tc0_formal}
     Given an input tensor $\X \in \R^{h \times w \times c}$. Let $\mathsf{Attn}(X):\R^{h \times w \times c} \to \R^{h \times w \times c}$ as defined in Definition~\ref{def:attn_layer}. Under the following constraints:
     \begin{itemize}
         \item Satisfy $h = w = n$,
         \item Channel bound: $c \leq n$,
         \item Precision bound: $p \leq \poly(n)$.
     \end{itemize}
     The $\mathsf{Attn}(\X)$ layer can be computed by a uniform $\mathsf{TC}^0$ circuit with:
     \begin{itemize}
        \item Size: $\poly(n)$.
        \item Depth: $6(d_{\mathrm{std}} + d_{\oplus}) + d_{\exp}$.
     \end{itemize}    
     with $d_{\mathrm{std}}$ and $d_{\oplus}$ defined in Definition~\ref{lem:float_operations_TC}, $d_{\exp}$ defined in Definition~\ref{lem:exp}. 
\end{lemma}
\begin{proof}
    We analyze the $\mathsf{TC}^0$ simulation of the attention layer through sequential computation phases:
    \begin{itemize}
        \item {\bf Key-Query Product}: Compute $W_QW_K^\top$ vial Lemma~\ref{lem:matrix_multi} requires depth $d_{\mathrm{std}} + d_{\oplus}$.
        \item {\bf Pairwise Score Computation}: Compute $s_{i,j} = X_{i,*}   W_Q   W_K^\top   X_{j,*}^\top$ requires depth $2(d_{\mathrm{std}} + d_{\oplus})$ by Lemma~\ref{lem:matrix_multi}. By  Lemma~\ref{lem:exp}, computing $A_{i,j} = \exp(s_{i,j})$ requires depth $d_{\exp}$. 
    \end{itemize}
    Since all entries $A_{i,j}$ for $i, j \in [n]$ can be computed in parallel, the attention matrix $A$ computation requires depth $3(d_{\mathrm{std}} + d_{\oplus}) + d_{\exp}$.

    Then keep on analyzing:
    \begin{itemize}
        \item {\bf Row Nomalization:} Computing $D:=\diag(A{\bf 1}_n)$ requires depth $d_{\oplus}$ by Lemma~\ref{lem:float_operations_TC}. Computing $D^{-1}$ requires depth $d_{\mathrm{std}}$ by Lemma~\ref{lem:float_operations_TC} .
        \item {\bf Value Projection} Computing $AXW_V$ requires depth $2(d_{\mathrm{std}} + d_{\oplus})$ by applying Lemma~\ref{lem:matrix_multi}. Compute $D^{-1} \cdot A X W_V$ requires $d_{\mathrm{std}}$.
    \end{itemize}

    Combing the result, we need a
    \begin{align*}
        d_{\mathrm{all}} = 6(d_{\mathrm{std}} + d_{\oplus}) + d_{\exp}
    \end{align*}
    depth and size $\poly(n)$ uniform $\mathsf{TC}^0$ circuit to compute the attention layer.
\end{proof}

\subsection{Computing Layer-wise Norm Layer in \texorpdfstring{$\TC^0$}{} }\label{sec:app_ln_tc0}
This section presents the detailed proof for Lemma~\ref{lem:ln_tc0_informal}.

\begin{lemma}[Layer-wise norm layer computation in $\TC^0$, formal version of Lemma~\ref{lem:ln_tc0_informal}]\label{lem:ln_tc0_formal}
    Given an input tensor $\X \in \R^{h \times w \times c}$. Let $\mathsf{LN}(X):\R^{h \times w \times c} \to \R^{h \times w \times c}$ as defined in Definition~\ref{def:ln}. Under the following constraints:
    \begin{itemize}
         \item Satisfy $h = w = n$,
         \item Channel bound: $c \leq n$,
         \item Precision bound: $p \leq \poly(n)$.
     \end{itemize}
     The $\mathsf{LN}(\X)$ layer can be computed by a uniform $\mathsf{TC}^0$ circuit with:
     \begin{itemize}
        \item Size: $\poly(n)$.
        \item Depth: $5d_\mathrm{std} + 2d_{\oplus} + d_\mathrm{sqrt}$.
     \end{itemize}    
     with $d_{\mathrm{std}}$ and $d_{\oplus}$ defined in Definition~\ref{lem:float_operations_TC}, $d_{\mathrm{sqrt}}$ defined in Definition~\ref{lem:sqrt}.
\end{lemma}

\begin{proof}
    By Part 1 and Part 3  of Lemma~\ref{lem:float_operations_TC}, 
    computing mean vector $\mu_j$ requires depth $d_{\mathrm{std}}+d_{\oplus}$. By Part 1 and Part 3  of Lemma~\ref{lem:float_operations_TC}, 
    computing mean vector $\sigma^2_i$ requires depth $2d_{\mathrm{std}}+d_{\oplus}$.      By Lemma~\ref{lem:float_operations_TC} and Lemma~\ref{lem:sqrt}, computing $Y_{j,*}$ requires depth  $2d_{\mathsf{std}}+d_{\oplus}$. So the process requires total depth $5d_\mathrm{std} + 2d_{\oplus} + d_\mathrm{sqrt}$ and $\poly(n)$ size. 
\end{proof}
\section{Provably Efficient Criteria}\label{sec:app_efficient_critieria}

\subsection{Running Time Analysis for Inference Pipeline of Origin FlowAR Architecture }\label{sec:runtime_origin_flowar}

We proceed to compute the total running time for the inference pipeline of the origin FlowAR architecture.
\begin{lemma}[Inference Runtime of Original FlowAR Architecture, formal version of Lemma~\ref{lem:runtime_old_flowar_informal}]\label{lem:runtime_old_flowar_formal}
    Consider the original FlowAR inference pipeline with the following parameters:
    \begin{itemize}
        \item {\bf Input tensor:} $\X \in \R^{h \times w \times c}$. Assume $h=w=n$ and $c = O(\log n)$.
        \item {\bf Number of scales:} $K = O(1)$.
        \item {\bf Scale factor:} For $i \in [K]$, $r_i:= a^{K-i}$ where base factor $a \in \mathbb{N}^+$.
        \item {\bf Upsampling functions}  For $i \in [K]$, $\phi_{\mathrm{up},i}(\cdot,a)$ from Definition~\ref{def:bicubic_up_sample_function}.
        \item {\bf Attention layer:}  For $i \in [K]$, $\mathsf{Attn}_i(\cdot)$ which acts on flattened sequences of dimension defined in Definition~\ref{def:attn_layer}.
        \item {\bf Feed forward layer: } For $i \in [K]$, $\mathsf{FFN}_i(\cdot)$ which acts on flattened sequences of dimension defined in Definition~\ref{def:ffn}.
        \item {\bf Flow matching layer:} For $i \in [K]$, $\mathsf{NN}_i(\cdot,\cdot,\cdot)$ denote the flow-matching layer defined in Definition~\ref{def:flow_matching_architecture}.
    \end{itemize}
    Under these conditions, the total inference runtime of FlowAR is bounded by $O(n^{4+o(1)})$.
    
\end{lemma}
\begin{proof}
    {\bf Part 1: Running time of bicubic up-sample Layer.} The $i$-th layer pf FlowAR model contains $\phi_{\mathrm{up},1}(\cdot,2),\dots,\phi_{\mathrm{up},i-1}(\cdot,2)$. Considering $\phi_{\mathrm{up},i-1}(\cdot,2)$, this operation needs $O(n^{2}c/2^{2(K-i)})$ time. Then the total time required for upsampling in the i-th layer of the FlowAR architecture is $O( n^2 c \cdot \frac{1}{2^{2K}} \cdot (1 - \frac{1}{4^i}))$ which is due to simple algebra. Hence, the total runtime for all bicubic up sample functions is
    \begin{align*}
        \mathcal{T}_{\mathrm{up}} = &~ \sum_{i=1}^K O( n^2 c \cdot \frac{1}{2^{2K}} \cdot (1 - \frac{1}{4^i}))\\
        =&~ O(n^{2+o(1)})
    \end{align*}
    where the first equation is derived from summing up all the running time of the up sample functions, the second step is due to simple algebra and $K = O(1)$ and $c = O(\log n)$.

    {\bf Part 2: Running time of Attention Layer.} The input size of the $i$-th attention layer $\mathsf{Attn}_i$ is $\sum_{j=1}^i (n/2^{K-j}) \times \sum_{j=1}^i(n/2^{K-j}) \times c $. So the time needed to compute the $i$-th attention layer is $O(n^4c \cdot (2^i-1)^4/2^{4K-4})$. Hence, the total runtime for all attention layers is
    \begin{align*}
        \mathcal{T}_{\mathrm{Attn}} =&~ \sum_{i=1}^K O( n^4c \cdot (2^i-1)^4/2^{4K-4})\\
        =&~ O(n^{4+o(1)})
    \end{align*}
    The first equation is derived from summing up all the running time of the attention layer, the second step is due to simple algebra and $K = O(1)$ and $c = O(\log n)$.

    {\bf Part 3: Running time of FFN Layer.} The input size of the $i$-th FFN layer $\mathsf{FFN}_i$ is $\sum_{j=1}^i (n/2^{K-j}) \times \sum_{j=1}^i(n/2^{K-j}) \times c $. So by Definition~\ref{def:ffn}, we can easily derive that the time needed to compute the $i-$th FFN layer is $O(n^2c^2 (2^i-1)^2/2^{2K-2} )$. Hence, the total runtime for all FFN layers is
    \begin{align*}
        \mathcal{T}_{\mathrm{FFN}} = &~ \sum_{i=1}^K O(n^2c^2 (2^i-1)^2/2^{2K-2}) \\
        =&~ O(n^{2+o(1)})
    \end{align*}
    The first step is derived from summing up all the running time of the FFN layer,  and the second step is due to simple algebra and $K = O(1)$ and $c = O(\log n)$.

    {\bf Part 4: Running time of Flow Matching Layer.} The input size of the $i$-th flow-matching layer $\mathsf{NN}_i$ is $ (n/2^{K-i}) \times (n/2^{K-i}) \times c $. It's trivially that the running time of the flow-matching layer will be dominated by the running time of the attention layer, which is $O(n^{4}c/ 2^{4(K-i)})$ (see Part 2 of Definition~\ref{def:flow_matching_architecture}). Hence, the total runtime for all flow-matching layers is 
    \begin{align*}
        \mathcal{T}_{\mathsf{FM}} =&~ \sum_{i=1}^K O(n^4 c /2^{4(K-i)})\\
        =&~ O(n^{4+o(1)})
    \end{align*}
    The first step is derived from summing up all the running time of the origin flow-matching layer, and the second step is due to simple algebra and $K = O(1)$ and $c = O(\log n)$.

    Then, by summing up Part 1 to Part 4, we can get the total running time for FlowAR architecture, which is
    \begin{align*}
        \mathcal{T}_{\mathrm{ori}} =&~ \mathcal{T}_{\mathrm{up}} + \mathcal{T}_{\mathrm{Attn}} + \mathcal{T}_{\mathrm{FFN}} + \mathcal{T}_{\mathsf{FM}}\\
        =&~ O(n^{4+o(1)})
    \end{align*}
\end{proof}
Lemma~\ref{sec:runtime_origin_flowar} demonstrates the runtime required for the original FlowAR architecture, from which we can deduce that the dominant term in the runtime comes from the computation of the Attention Layer.
\subsection{Running Time Analysis for Inference Pipeline of Fast FlowAR Architecture }\label{sec:runtime_fast_flowar}
In this section, we apply the conclusions of \cite{as23} to the FlowAR architecture, where all Attention modules in FlowAR are computed using the Approximate Attention Computation defined in Definition~\ref{def:aattc}.
\begin{lemma}[Inference Runtime of Fast FlowAR Architecture, formal version of Lemma~\ref{lem:runtime_fast_flowar_informal}]\label{lem:runtime_fast_flowar_formal}
    Consider the original FlowAR inference pipeline with the following parameters:
    \begin{itemize}
        \item {\bf Input tensor:} $\X \in \R^{h \times w \times c}$. Assume $h=w=n$ and $c = O(\log n)$.
        \item {\bf Number of scales:} $K = O(1)$.
        \item {\bf Scale factor:} For $i \in [K]$, $r_i:= a^{K-i}$ where base factor $a \in \mathbb{N}^+$.
        \item {\bf Upsampling functions}  For $i \in [K]$, $\phi_{\mathrm{up},i}(\cdot,a)$ from Definition~\ref{def:bicubic_up_sample_function}.
        \item {\bf Approximate Attention layer:}  For $i \in [K]$, $\mathsf{AAttC}_i(\cdot)$ defined in Definition~\ref{def:aattc}.
        \item {\bf Feed forward layer: } For $i \in [K]$, $\mathsf{FFN}_i(\cdot)$ which acts on flattened sequences of dimension defined in Definition~\ref{def:ffn}.
        \item {\bf Fast flow-matching layer:} For $i \in [K]$, $\mathsf{FNN}_i(\cdot,\cdot,\cdot)$ denote the fast flow-matching layer defined in Definition~\ref{def:fast_flow_matching_architecture}.
    \end{itemize}
    Under these conditions, the total inference runtime of FlowAR is bounded by $O(n^{2+o(1)})$.
      
\end{lemma}
\begin{proof}
    {\bf Part 1: Running time of bicubic up-sample Layer.} The runtime of the upsample function in the fast FlowAR architecture is the same as that in the original FlowAR architecture, which is
    \begin{align*}
        \mathcal{T}_{\mathrm{up}} =  O(n^{2+o(1)})
    \end{align*}

    {\bf Part 2: Running time of Attention Layer.} The input size of the $i$-th approximate attention computation layer $\mathsf{AAttC}_i$ is $\sum_{j=1}^i (n/2^{K-j}) \times \sum_{j=1}^i(n/2^{K-j}) \times c $. So the time needed to compute the $i$-th attention layer is $O(n^{2+o(1)} \cdot (2^i-1)^4/2^{4K-4})$. Hence, the total runtime for all attention layers is
    \begin{align*}
        \mathcal{T}_{\mathrm{Attn}} =&~ \sum_{i=1}^K O( n^{2+o(1)} \cdot (2^i-1)^4/2^{4K-4})\\
        =&~ O(n^{2+o(1)})
    \end{align*}
    The first equation is derived from summing up all the running time of the approximate attention computation layer, and the second equation is due to basic algebra and $K = O(1)$.

    {\bf Part 3: Running time of FFN Layer.} The runtime of the FFN layer in the fast FlowAR architecture is the same as that in the original FlowAR architecture, which is
    \begin{align*}
        \mathcal{T}_{\mathrm{FFN}} =  O(n^{2+o(1)})
    \end{align*}

    {\bf Part 4: Running time of Flow Matching Layer.} For each $i \in [K]$, the input size of the $i$-th fast flow-matching layer $\mathsf{FNN}_i$ is $ (n/2^{K-i}) \times (n/2^{K-i}) \times c $. By Definition~\ref{def:mlp}, we can know that the total computational time for the MLP layer is $O(n^{2+o(1)})$, which is due to $c=O(\log n)$. Then by Lemma~\ref{lem:as23_attention}, we can speed up the attention computation from $O(n^{4+o(1)})$ to $O(n^{2+o(1)})$. Hence, the total runtime for all flow-matching layers is
     \begin{align*}
        \mathcal{T}_{\mathrm{Attn}} =&~ \sum_{i=1}^K O( n^{2+o(1)})\\
        =&~ O(n^{2+o(1)})
    \end{align*}
    The equation is due to $K = O(1)$.

    Then, by summing up Part 1 to Part 4, we can get the total running time for fast FlowAR architecture, which is
    \begin{align*}
        \mathcal{T}_{\mathrm{fast}} =&~ \mathcal{T}_{\mathrm{up}} + \mathcal{T}_{\mathrm{Attn}} + \mathcal{T}_{\mathrm{FFN}} + \mathcal{T}_{\mathsf{FM}}\\
        =&~ O(n^{2+o(1)})
    \end{align*}

\end{proof}

\subsection{Error Analysis of \texorpdfstring{$\mathsf{MLP}(\X')$}{} and \texorpdfstring{$\mathsf{MLP}(\X)$}{}}\label{sec:error_analysis_of_mlp_x_prime_mlp_x}
We conduct the error analysis between $\mathsf{MLP}(\X')$ and $\mathsf{MLP}(\X)$ where $\X'$ is the approximation version of $\X$.
\begin{lemma}[Error analysis of MLP Layer]\label{lem:error_analysis_mlp}
    If the following conditions hold:
    \begin{itemize}
        \item Let $\X \in \R^{h \times w \times c}$ denote the input tensor.
        \item Let $\X' \in \R^{h \times w \times c}$ denote the approximation version of input tensor $\X$.
        \item Let $\epsilon \in (0, 0.1)$ denote the approximation error. 
        \item Suppose we have $\| \X' - \X \|_\infty \leq \epsilon$.
        \item Let $R > 1$.
        \item Assume the value of each entry in matrices can be bounded by $R$.  
        \item Let $\mathsf{MLP}(\cdot,c,d)$ denote the MLP layer defined in Definition~\ref{def:mlp}.
    \end{itemize}
    We can demonstrate the following
    \begin{align*}
        \|\mathsf{MLP}(\X') - \mathsf{MLP}(\X)\|_\infty \leq cR\epsilon
    \end{align*}
    Here, we abuse the $\ell_\infty$ norm in its tensor form for clarity.
\end{lemma}
\begin{proof}
    We can show that for $i \in [h],j \in [w], l \in [c]$, we have
    \begin{align*}
        \|\mathsf{MLP}(\X',c,d)_{i,j,*} - \mathsf{MLP}(\X,c,d)_{i,j,*}\|_\infty =&~ \| \X'_{i,j,*}\cdot W - \X_{i,j,*} \cdot W \|_\infty\\
        \leq&~ \|\underbrace{(\X'_{i,j,*}-\X_{i,j,*})}_{1 \times c} \cdot \underbrace{W}_{c\times d} \|_\infty \\
        \leq &~ c \cdot \|\underbrace{(\X'_{i,j,*}-\X_{i,j,*})}_{1 \times c}\|_\infty \cdot \|\underbrace{W}_{c\times d} \|_\infty\\
        \leq&~ c \cdot R \cdot \epsilon
    \end{align*}
    The first equation is due to Definition~\ref{def:mlp}, the second inequality is derived from simple algebra, the third inequality is a consequence of basic matrix multiplication, and the last inequality comes from the conditions of this lemma.

    Then by the definition of $\ell_\infty$ norm, we can easily get the proof.
\end{proof}

\subsection{Error Analysis of \texorpdfstring{$\mathsf{AAttC}(\X')$}{} and \texorpdfstring{$\mathsf{Attn}(\X)$}{}}\label{sec:error_analysis_of_aattc_x_prime_attn_x}
We conduct the error analysis between $\mathsf{AAttC}(\X')$ and $\mathsf{Attn}(\X)$ where $\X'$ is the approximation version of $\X$.
\begin{lemma}[Error analysis of $\mathsf{AAttC}(X')$ and $\mathsf{Attn}(X)$, Lemma B.4 of \cite{kll+25}]\label{lem:error_analysis_aattc_attn}
    If the following conditions hold:
    \begin{itemize}
        \item Let $\X \in \R^{h \times w \times c}$ denote the input tensor.
        \item Let $\X' \in \R^{h \times w \times c}$ denote the approximation version of input tensor $\X$.
        \item Let $\epsilon \in (0, 0.1)$ denote the approximation error. 
        \item Suppose we have $\| \X' - \X \|_\infty \leq \epsilon$.
        \item Let $R > 1$.
        \item Assume the value of each entry in matrices can be bounded by $R$. 
        \item Let $\mathsf{Attn}$ denote the attention layer defined in Definition~\ref{def:attn_layer}.
        \item Let $\mathsf{AAttC}$ denote the approximated attention layer defined in Definition~\ref{def:aattc}.
        \item Let $U,V \in \R^{hw \times k}$ be low-rank matrices constructed for polynomial approximation of attention matrix $\mathsf{AAttC}(\X)$.
        \item Let $f$ be a polynomial with degree $g$.
    \end{itemize}
    We can demonstrate the following:
    \begin{align*}
        \| \mathsf{AAttC}(\X') - \mathsf{Attn}(\X) \|_\infty \leq O( k R^{g+1} c) \cdot \epsilon
    \end{align*}
    Here, we abuse the $\ell_\infty$ norm in its tensor form for clarity.
\end{lemma}

\subsection{Error Analysis of \texorpdfstring{$\mathsf{FFN}(\X')$}{} and \texorpdfstring{$\mathsf{FFN}(\X)$}{}}\label{sec:error_analysis_of_ffn_x_prime_ffn_x}
In this section, we conduct the error analysis between $\mathsf{FFN}(\X')$ and $\mathsf{FFN}(\X)$ where $\X'$ is the approximation version of $\X$.
\begin{lemma}[Error analysis of $\mathsf{FFN}(\X')$ and $\mathsf{FFN}(\X)$]\label{lem:error_analysis_ffn}
    If the following conditions hold:
    \begin{itemize}
        \item Let $\X \in \R^{h \times w \times c}$ denote the input tensor.
        \item Let $\X' \in \R^{h \times w \times c}$ denote the approximation version of input tensor $\X$.
        \item Let $\epsilon \in (0, 0.1)$ denote the approximation error. 
        \item Suppose we have $\| \X' - \X \|_\infty \leq \epsilon$.
        \item Let $R > 1$.
        \item Assume the value of each entry in matrices can be bounded by $R$. 
        \item Let $\mathsf{FFN}$ denote the FFN layer defined in Definition~\ref{def:ffn}.
        \item Let the activation function $\sigma(\cdot)$ in $\mathsf{FFN}$ be the ReLU activation function.
    \end{itemize}
    We can demonstrate the following:
    \begin{align*}
        \| \mathsf{FFN}(\X') - \mathsf{FFN}(\X) \|_\infty \leq O(c^2 R^2) \cdot \epsilon
    \end{align*}
    Here, we abuse the $\ell_\infty$ norm in its tensor form for clarity.
\end{lemma}
\begin{proof}
    Firstly we can bound that for $i \in [h], j \in [w]$
    \begin{align}\label{eq:linear_transformation_bound}
        \| (\X'_{i,j,*}\cdot W_1 +b_1) - (\X_{i,j,*}\cdot W_1 +b_1)\|_\infty =&~ \| \underbrace{(\X'_{i,j,*}-\X_{i,j,*})}_{1 \times c} \cdot \underbrace{W_1}_{c\times c}\|_\infty\notag\\
        \leq&~ c \cdot \|\X'_{i,j,*}-\X_{i,j,*}\|_\infty \|W_1 \|_\infty\notag\\
        \leq&~ c \cdot \epsilon \cdot R
    \end{align}
    The first equation comes from basic algebra, the second inequality is due to basic matrix multiplication, and the last inequality follows from the conditions of this lemma.
    
    We can show that for $i \in [h], j \in [w]$,
    \begin{align*}
        &~\| \mathsf{FFN}(\X')_{i,j,*} - \mathsf{FFN}(\X)_{i,j,*} \|_\infty\\=&~ \| \X'_{i,j,*}-\X_{i,j,*} +\underbrace{ (\sigma(\X_{i,j,*}\cdot W_1 + b_1)- \sigma(\X'_{i,j,*}\cdot W_1 + b_1))}_{1 \times c} \cdot \underbrace{W_2}_{c\times c}\|_\infty \\
        \leq&~ \| \X'_{i,j,*}-\X_{i,j,*}\|_\infty +c\cdot \|W_2\|_\infty \cdot \|\sigma(\X_{i,j,*}\cdot W_1 + b_1)- \sigma(\X'_{i,j,*}\cdot W_1 + b_1)\|_\infty\\
        \leq&~ \epsilon + c R \cdot \| (\X'_{i,j,*} W_1 +b_1) - (\X_{i,j,*} W_1 +b_1)\|_\infty\\
        \leq&~ \epsilon + c^2 R^2 \cdot \epsilon \\
        =&~ O(c^2 R^2) \cdot \epsilon
    \end{align*}
    The first equation is due to Definition~\ref{def:ffn}, the second step follows from triangle inequality and basic matrix multiplication, the third step follows from the property of ReLU activation function and basic algebra, the fourth step follows from Eq.~\eqref{eq:linear_transformation_bound}, and the last step follows from simple algebra.
\end{proof}

\subsection{Error Analysis of \texorpdfstring{$\phi_{\mathrm{up}}(\X')$}{} and \texorpdfstring{$\phi_{\mathrm{up}}(\X)$}{}}\label{sec:error_analysis_of_phi_x_prime_phi_x}
In this section, we conduct the error analysis between $\phi_{\mathrm{up}}(\X')$ and $\phi_{\mathrm{up}}(\X)$ where $\X'$ is the approximation version of $\X$.
\begin{lemma}[Error Analysis of Up Sample Layer, Lemma B.5 of \cite{kll+25}]\label{lem:error_analysis_up_layer}
If the following conditions hold:
\begin{itemize}
    \item Let $\X \in \R^{h \times w \times c}$ denote the input tensor.
    \item Let $\X' \in \R^{h \times w \times c}$ denote the approximation version of input tensor $\X$.
    \item Let $a = 2$ denote a positive integer.
    \item Let $\phi_{\mathrm{up}, i}(\cdot, a)$ be the bicubic up sample function defined in Definition~\ref{def:bicubic_up_sample_function}.
    \item Let $\epsilon \in (0,0.1)$ denote the approximation error.
    \item Let $\|X-X'\|_\infty\leq \epsilon$.
\end{itemize}
Then we have
\begin{align*}
    \| \phi_{\rm up}(\X',a) - \phi_{\rm up}(\X,a) \|_\infty \leq O(\epsilon)
\end{align*}
Here, we abuse the $\ell_\infty$ norm in its tensor form for clarity.
\end{lemma}

\subsection{Error Analysis of \texorpdfstring{$\mathsf{FNN}(\F'^t,\X',t)$}{} and \texorpdfstring{$\mathsf{NN}(\F^t,\X,t)$}{}}\label{sec:error_analysis_of_flow_matching_layer}
In this section, we conduct the error analysis between $\mathsf{FNN}(\F'^t,\X',t)$ and  $\mathsf{NN}(\F^t,\X,t)$ where $\X'$ is the approximation version of $\X$.
\begin{lemma}[Error Analysis of Flow Matching Layer]\label{lem:error_analysis_flow_matching_layer}
If the following conditions hold:
\begin{itemize}
    \item Let $\X \in \R^{h \times w \times c}$ denote the input tensor.
    \item Let $\X' \in \R^{h \times w \times c}$ denote the approximation version of input tensor $\X$.
    \item Let $\F^t,\mathsf{FF}^t \in \R^{h \times w \times c}$ be the interpolated input defined in Definition~\ref{def:flow}.
    \item Let $\mathsf{NN}(\cdot,\cdot,\cdot)$ denote flow-matching layer defined in Definition~\ref{def:flow_matching_architecture}.
    \item Let $\mathsf{FNN}(\cdot,\cdot,\cdot)$ denote fast flow-matching layer defined in Definition~\ref{def:fast_flow_matching_architecture}.
     \item Let $\mathsf{Attn}$ denote the attention layer defined in Definition~\ref{def:attn_layer}.
    \item Let $\mathsf{AAttC}$ denote the approximated attention layer defined in Definition~\ref{def:aattc}.
    \item Let $R > 1$.
    \item Assume the value of each entry in matrices can be bounded by $R$. 
    \item Let $U,V \in \R^{hw \times k}$ be low-rank matrices constructed for polynomial approximation of attention matrix $\mathsf{AAttC}(\X)$.
    \item Let $f$ be a polynomial with degree $g$.
    \item Let $\epsilon \in (0,0.1)$ denote the approximation error.
    \item Let $\|\X-\X'\|_\infty\leq \epsilon$.
    \item Let $t \in [0,1]$ denote a time step.
    \item Assume that Layer-wise Norm layer $\mathsf{LN}(\cdot)$ defined in Definition~\ref{def:ln} does not exacerbate the propagation of errors, i.e., if $\|X'-X\|_\infty \leq \epsilon$, then $\|\mathsf{LN}(X')-\mathsf{LN}(X)\|_\infty \leq \epsilon$.
\end{itemize}
Then we have
\begin{align*}
    \| \mathsf{FNN}(\mathsf{FF}^t,\X',t) - \mathsf{NN}(\F^t,\X,t) \|_\infty \leq O(kR^{g+6}c^3) \cdot \epsilon
\end{align*}
Here, we abuse the $\ell_\infty$ norm in its tensor form for clarity.
\end{lemma}
\begin{proof}
    Firstly, we can show that
    \begin{align*}
        \|\mathsf{FF}^t- \F^t\|_\infty = \| t (\X' - \X) \|_\infty \leq \epsilon
    \end{align*}
    The inequality comes from $t \in [0,1]$ and $\|\X'-\X\|_\infty \leq \epsilon$.

    By {\bf Step 1} of Definition~\ref{def:flow_matching_architecture} and Definition~\ref{def:fast_flow_matching_architecture}, we need to compute
    \begin{align*}
        \alpha_1, \alpha_2, \beta_1, \beta_2, \gamma_1, \gamma_2=&~  \mathsf{MLP}(\X + t \cdot {\bf 1}_{h \times w \times c},c,6c)\\
         \alpha'_1, \alpha'_2, \beta'_1, \beta'_2, \gamma'_1, \gamma'_2=&~  \mathsf{MLP}(\X' + t \cdot {\bf 1}_{h \times w \times c},c,6c)\\
    \end{align*}
    Then, we can show that
    \begin{align*}
        \|\alpha'_1 - \alpha_1\|_\infty \leq c R \epsilon
    \end{align*}
    where the step follows from Lemma~\ref{lem:error_analysis_mlp}. The same conclusion holds for the intermediate parameter $\alpha_2, \beta_1, \beta_2, \gamma_1, \gamma_2$.

    By {\bf Step 2} of Definition~\ref{def:flow_matching_architecture} and Definition~\ref{def:fast_flow_matching_architecture}, we need to compute
    \begin{align*}
        \F'^t =&~ \mathsf{Attn}(\gamma_1 \circ \mathsf{LN}(\F^t)+\beta_1) \circ \alpha_1\\
        \mathsf{FF}'^t =&~ \mathsf{AAttC}(\gamma'_1 \circ \mathsf{LN}(\mathsf{FF}^t)+\beta'_1) \circ \alpha'_1\\
    \end{align*}
    Then, we move forward to show that
    \begin{align}\label{eq:flow_matching_tmp1}
        &~\| \gamma'_1 \circ \mathsf{LN}(\mathsf{FF}^t) + \beta'_1 - \gamma_1 \circ \mathsf{LN}(\F^t) - \beta_1\|_\infty\notag\\
        \leq&~ \| \gamma'_1 \circ \mathsf{LN}(\mathsf{FF}^t) - \gamma_1 \circ \mathsf{LN}(\F^t)  \|_\infty + \|\beta'_1-\beta_1\|_\infty\notag\\
        \leq&~  \| \gamma'_1 \circ (\mathsf{LN}(\mathsf{FF}^t) - \mathsf{LN}(\F^t))\| + \|(\gamma'_1-\gamma_1) \circ \mathsf{LN}(\F^t)\|_\infty + cR\epsilon\notag\\
        \leq&~ R \cdot \epsilon + R\cdot \epsilon + cR\epsilon\notag\\
        =&~ O(cR)\cdot \epsilon
    \end{align}
    where the first and second step follows from triangle inequality, the third step follows from conditions of this Lemma, and the last step follows from simple algebra.

    Then we have
    \begin{align}\label{eq:flow_matching_tmp2}
        \| \mathsf{AAttC}(\gamma'_1 \circ \mathsf{LN}(\mathsf{FF}^t) + \beta'_1) -\mathsf{Attn}(\gamma_1 \circ \mathsf{LN}(\F^t) + \beta_1)\|_\infty \leq&~ O(k R^{g+1} c) \cdot O(cR) \cdot \epsilon \notag\\
        \leq&~ O(kR^{g+2} c^2) \epsilon
    \end{align}
    where the first step follows from Lemma~\ref{lem:error_analysis_aattc_attn} and Eq.~\eqref{eq:flow_matching_tmp1} and the second step follows from simple algebra.

    Now, we are able to show that
    \begin{align}\label{eq:error_analysis_tmp3}
        \| \mathsf{FF}'^t -  \F'^t\|_\infty =&~ \|\mathsf{AAttC}(\gamma'_1 \circ \mathsf{LN}(\mathsf{FF}^t)+\beta'_1) \circ \alpha'_1 -  \mathsf{Attn}(\gamma_1 \circ \mathsf{LN}(\F^t)+\beta_1) \circ \alpha_1\|_\infty\notag \\
        \leq&~ \| \mathsf{AAttC}(\gamma'_1 \circ \mathsf{LN}(\mathsf{FF}^t)+\beta'_1) \circ (\alpha'_1-\alpha_1)\|_\infty \notag \\+&~ \|\alpha_1 \cdot \mathsf{AAttC}(\gamma'_1 \circ \mathsf{LN}(\mathsf{FF}^t) + \beta'_1) -\mathsf{Attn}(\gamma_1 \circ \mathsf{LN}(\F^t) + \beta_1) \|_\infty\notag \\
        \leq&~ R \cdot cR\epsilon + R \cdot O(kR^{g+2}c^2) \epsilon\notag \\
        =&~   O(kR^{g+3}c^2) \epsilon
    \end{align}
    where the first step follows from the definition of $\wh{\F}'^t$ and $\wh{\F}^t$, the second step follows from triangle inequality, the third step follows from Eq.~\eqref{eq:flow_matching_tmp2} and the conditions of this lemma, and the last step follows from simple algebra.

    By {\bf Step 3} of Definition~\ref{def:flow_matching_architecture} and Definition~\ref{def:fast_flow_matching_architecture}, we need to compute
    \begin{align*}
        \F''^t =&~\mathsf{MLP}(\gamma_2 \circ \mathsf{LN}(\F'^t)+ \beta_2,c,c) \circ \alpha_2\\
        \mathsf{FF}''^t=&~\mathsf{MLP}(\gamma'_2 \circ \mathsf{LN}(\mathsf{FF}'^t)+ \beta'_2,c,c) \circ \alpha'_2
    \end{align*}
    Then, we move forward to show that
    \begin{align}\label{eq:error_analysis_tmp4}
        &~\|\gamma'_2 \circ \mathsf{LN}(\mathsf{FF}'^t)+ \beta'_2 - \gamma_2 \circ \mathsf{LN}(\F'^t) -\beta_2 \|_\infty\notag\\
        \leq&~ \| \gamma'_2 \circ \mathsf{LN}(\mathsf{FF}'^t)- \gamma_2 \circ \mathsf{LN}(\F'^t) \|_\infty + \|\beta'_1-\beta_1\|_\infty\notag\\
        \leq&~\|\gamma'_2\circ(\mathsf{LN}(\mathsf{FF}'^t) - \mathsf{LN}(\F'^t)) \|_\infty+ \|(\gamma'_2-\gamma_2)\circ \mathsf{LN}(\F'^t)\|_\infty + cR\epsilon\notag\\
        \leq&~ R \cdot O(kR^{g+3}c^2) \epsilon + cR\epsilon \cdot R + cR\epsilon\notag\\
        =&~ O(kR^{g+4}c^2)\cdot \epsilon
    \end{align}
    where the first and the second steps follow from triangle inequality, the third step follows from Eq.~\eqref{eq:error_analysis_tmp3} the conditions of this lemma, and the last step follows from simple algebra.

    Then, we can show
    \begin{align}\label{eq:error_analysis_tmp5}
        \|\mathsf{MLP}(\gamma'_2 \circ \mathsf{LN}(\mathsf{FF}'^t)+ \beta'_2)-\mathsf{MLP}(\gamma_2 \circ \mathsf{LN}(\F'^t) +\beta_2)\|_\infty \leq&~  c R \cdot O(kR^{g+4} c^2) \cdot \epsilon\notag\\
        =&~ O(kR^{g+5}c^3)\cdot \epsilon
    \end{align}
    where the first step follows from Lemma~\ref{lem:error_analysis_mlp} and Eq.~\eqref{eq:error_analysis_tmp4} and the second step follows from simple algebra.

    Finally, we are able to show that
    \begin{align*}
        &~\|\mathsf{FNN}(\mathsf{FF}^t,\X',t) - \mathsf{FN}(\F^t,\X,t)\|_\infty\\
        =&~ \|\mathsf{MLP}(\gamma'_2 \circ \mathsf{LN}(\mathsf{FF}'^t)+ \beta'_2,c,c) \circ \alpha'_2 - \mathsf{MLP}(\gamma_2 \circ \mathsf{LN}(\F'^t)+ \beta_2,c,c) \circ \alpha_2 \|_\infty\\
        \leq&~ \| (\mathsf{MLP}(\gamma'_2 \circ \mathsf{LN}(\mathsf{FF}'^t)+ \beta'_2,c,c) - \mathsf{MLP}(\gamma_2 \circ \mathsf{LN}(\F'^t)+ \beta_2,c,c))\circ \alpha'_2\|_\infty \\+ &~ \|\mathsf{MLP}(\gamma_2 \circ \mathsf{LN}(\F'^t)+ \beta_2,c,c)  \circ (\alpha'_2 - \alpha_2) \|_\infty\\
        \leq&~ R \cdot O(kR^{g+5}c^3) \cdot \epsilon + R \cdot cR\epsilon\\
        =&~  O(kR^{g+6}c^3) \cdot \epsilon
    \end{align*}
    where the step follows from the definition of output of $\mathsf{FFN}(\F'^t,\X',t)$ and $ \mathsf{FN}(\F^t,\X,t)$, the second step follows from triangle inequality, the third step follows from Eq.~\eqref{eq:error_analysis_tmp5} and conditions of this lemma, and the last step follows from simple algebra.

    Then, we complete the proof.
\end{proof}

\subsection{Error Analysis of Fast FlowAR Architecture}\label{sec:error_analysis_fast_flowar}
Here, we proceed to present the error analysis of fast FlowAR Architecture.
\begin{lemma}[Error Bound Between Fast FlowAR and FlowAR Outputs]\label{lem:error_analysis_fast_flowar}
    Given the following:
    \begin{itemize}
        \item {\bf Input tensor:} $\X \in \R^{h \times w \times c}$.
        \item {\bf Scales number:} $K = O(1)$.
        \item {\bf Dimensions:} Let $h=w=n$ and $c = O(\log n)$. Let $\wt{h}_i := \sum_{j=1}^i h/r_j$ and $\wt{w}_i := \sum_{j=1}^i w/r_j$.
        \item {\bf Bounded Entries:} All tensors and matrices have entries bounded by $R = O(\sqrt{\log n})$.
        \item {\bf Layers:}
        \begin{itemize}
            \item $\phi_{\mathrm{up},a}(\cdot)$ :  bicubic upsampling function (Definition~\ref{def:bicubic_up_sample_function}).
            \item $\mathsf{Attn}(\cdot)$: attention layer (Definition~\ref{def:attn_layer}).
            \item $\mathsf{AAttC(\cdot)}$: approximate attention layer (Definition~\ref{def:aattc})
            \item $\mathsf{NN}(\cdot,\cdot,\cdot)$: flow-matching layer (Definition~\ref{def:flow_matching_architecture})
            \item $\mathsf{FNN}(\cdot,\cdot,\cdot)$: fast flow-matching layer (Definition~\ref{def:fast_flow_matching_architecture})
        \end{itemize}
        \item {\bf Input and interpolations:}
        \begin{itemize}
            \item Initial inputs: $\Z_{\mathrm{init}} \in \R^{(h/r_1)\times(w/r_1) \times c}$.
            \item $\Z_i:$ Reshaped tensor of  $\Z_{\mathrm{init}}, \phi_{\mathrm{up},1}(\wt{\Y}_1), \dots, \phi_{\mathrm{up},i-1}(\wt{\Y}_{i-1})$ for FlowAR.
            \item $\Z'_i:$ Reshaped tensor of  $\Z_{\mathrm{init}}, \phi_{\mathrm{up},1}(\wt{\Y}'_1), \dots, \phi_{\mathrm{up},i-1}(\wt{\Y}'_{i-1})$ for  Fast FlowAR.
            \item $\mathsf{F}_i^{t_i} \in \R^{h/r_i \times w/r_i \times c}$ be the interpolated value of FlowAR (Definition~\ref{def:flow}).
            \item $\mathsf{FF}_i^{t_i} \in \R^{h/r_i \times w/r_i \times c}$ be the interpolated value of Fast FlowAR (Definition~\ref{def:flow}).
        \end{itemize} 
        \item {\bf Outputs:}
        \begin{itemize}
            \item $\wt{\Y}_i \in \R^{h/r_i \times w/r_i \times c}$: FlowAR output at layer $i$ (Definition~\ref{def:flow_architecture_inference})
            \item $\wt{\Y}'_i \in \R^{h/r_i \times w/r_i \times c}$: Fast FlowAR output at layer $i$ (Definition~\ref{def:fast_flow_architecture_inference})
        \end{itemize}
    \end{itemize}
    Under these conditions, the $\ell_\infty$ error between the final outputs is bounded by:
    \begin{align*}
        \|\wt{\Y}'_K - \wt{\Y}_K\|_\infty \leq 1/\poly(n)
    \end{align*}
    
\end{lemma}
\begin{proof}
    We can conduct math induction as the following.
    
    Consider the first layer of fast FlowAR Architecture. Firstly, we can show that
    \begin{align*}
        \| \mathsf{AAttC}_1(\Z_{1}) - \mathsf{Attn}_1(\Z_{1})\|_\infty \leq 1/\poly(n)
    \end{align*}
    The inequality is derived Lemma~\ref{lem:as23_attention}.
    
    Then, we have
    \begin{align*}
        \| \wh{\Y}'_1 - \wh{\Y}_1\|_\infty =&~ \|\mathsf{FFN}_1(\mathsf{AAttC}_1(\Z_{1})) - \mathsf{FFN}_1(\mathsf{Attn}_1(\Z_{1}))\|_\infty\\
        \leq&~ O(c^2 R^2) \cdot 1/\poly(n)\\
        = &~ 1/\poly(n)
    \end{align*}
    The first equation comes from the definition of $\wh{Y}'_1$ and $\wh{Y}_1$, the second inequality is due to Lemma~\ref{lem:error_analysis_ffn} and the last equation is due to $c = O(\log n)$ and $R = O(\sqrt{\log n})$.

    Then, we can show that
    \begin{align*}
        \| \wt{\Y}'_1 - \wt{Y}_1\|_\infty =&~ \| \mathsf{FNN}_1(\mathsf{FF}^{t_1}_1,\wh{Y}'_1,t_1)  - \mathsf{NN}_1(\F^{t_1}_1,\wh{Y}_1,t_1) \|_\infty\\
        \leq &~ O(kR^{g+6}c^3) \cdot 1/\poly(n)\\
        =&~ 1/\poly(n)
    \end{align*}
    The first equation is due to the definition of $\Y'_1$ and $\Y_1$, the second inequality comes from Lemma~\ref{lem:error_analysis_flow_matching_layer}, and the last step follows from $c = O(\log n)$ and $R = O(\sqrt{\log n})$.

    Assume that the following statement is true for $k$-th iteration (where $k < K$):
    \begin{align*}
        \|\wt{\Y}'_k - \wt{Y}_k\|_\infty \leq 1/\poly(n)
    \end{align*}
    Then, we can easily to bound
    \begin{align*}
        \| \Z'_{k+1} - \Z_{k+1}\|_\infty \leq 1/\poly(n)
    \end{align*}
    The inequality is due to Lemma~\ref{sec:error_analysis_of_phi_x_prime_phi_x} and Definition of $\Z'_{k+1}$ and $\Z_{k+1}$.

    Then, we can show that 
    \begin{align*}
        \|\mathsf{AAttC}_{k+1}(\Z'_{k+1}) -\mathsf{Attn}_{k+1}(\Z_{k+1}) \|_\infty \leq&~ O(k R^{g+1} c) \cdot 1/\poly(n)\\
        =&~ 1/\poly(n) 
    \end{align*}
    The first inequality comes from Lemma~\ref{lem:error_analysis_aattc_attn}, and the second equation is due to $c = O(\log n)$ and $R = O(\sqrt{\log n})$.

    Then we have
    \begin{align*}
        \| \wh{\Y}'_{k+1} - \wh{\Y}_{k+1}\|_\infty =&~ \|\mathsf{FFN}_{k+1}(\mathsf{AAttC}_{k+1}(\Z'_{k+1})) - \mathsf{FFN}_{k+1}(\mathsf{Attn}_{k+1}(\Z_{k+1}))\|_\infty\\
        \leq&~ O(c^2 R^2) \cdot 1/\poly(n)\\
        = &~ 1/\poly(n)
    \end{align*}
   The first equation comes from the definition of $\wh{Y}'_{k+1}$ and $\wh{Y}_{k+1}$, the second inequality is due to Lemma~\ref{lem:error_analysis_ffn} and the third equation is due to $c = O(\log n)$ and $R = O(\sqrt{\log n})$.

    Then, we can derive that
    \begin{align*}
        \| \wt{\Y}'_{k+1} - \wt{Y}_{k+1}\|_\infty =&~ \| \mathsf{FNN}_{k+1}(\mathsf{FF}^{t_{k+1}}_{k+1},\wh{Y}'_{k+1},t_{k+1})  - \mathsf{NN}_{k+1}(\F^{t_{k+1}}_{k+1},\wh{Y}_k+1,t_{k+1}) \|_\infty\\
        \leq &~ O(kR^{g+6}c^3) \cdot 1/\poly(n)\\
        =&~ 1/\poly(n)
    \end{align*}
    The first equation comes from the definition of $\Y'_{k+1} $ and $ \Y_{k+1}$, the second inequality is due to Lemma~\ref{lem:error_analysis_flow_matching_layer} and the third equation is due to $c = O(\log n)$ and $R = O(\sqrt{\log n})$.

    Then, by mathematical induction, we can get the proof.
\end{proof}

\end{document}